\documentclass[11pt]{article}
\usepackage[margin=1 in]{geometry}
\usepackage{natbib}
\usepackage{graphicx}
\usepackage[dvipsnames]{xcolor}

\newcommand{\cM}{{\cal M}}
\newcommand{\mB}{{\bf B}}

\newcommand{\squeeze}{}
\newcommand{\notneeded}[1]{#1}

\usepackage[utf8]{inputenc} % allow utf-8 input
\usepackage[T1]{fontenc}    % use 8-bit T1 fonts
\usepackage{hyperref}       % hyperlinks
\usepackage{url}            % simple URL typesetting
\usepackage{booktabs}       % professional-quality tables
\usepackage{amsfonts}       % blackboard math symbols
\usepackage{nicefrac}       % compact symbols for 1/2, etc.
\usepackage{microtype}      % microtypography
\usepackage{xcolor}         % colors

\usepackage[colorinlistoftodos,bordercolor=orange,backgroundcolor=orange!20,linecolor=orange,textsize=scriptsize]{todonotes}

\usepackage{xspace}
\newcommand{\algname}[1]{{\sf \footnotesize #1}\xspace}

\input{preamble_neurips.tex}

\title{\bf EF21: A New, Simpler, Theoretically Better, \\ \bf and Practically Faster Error Feedback}

\author{Peter Richt\'arik \\ KAUST\thanks{King Abdullah University of Science and Technology, Thuwal, Saudi Arabia.} \and Igor Sokolov \\ KAUST \and  Ilyas Fatkhullin \\ KAUST\thanks{This paper was written while I.F.\ was an intern at KAUST.}\;  \& TU Munich}

\date{June 8, 2021}

\begin{document}
	
	\maketitle
	
	\begin{abstract}
	Error feedback (\algname{EF}), also known as error compensation, is an immensely popular convergence stabilization mechanism in the context of distributed training of supervised machine learning models enhanced by the use of contractive communication compression mechanisms, such as Top-$k$. First proposed by \citet{Seide2014} as a heuristic,  \algname{EF}  resisted any theoretical understanding until recently \citep{Stich-EF-NIPS2018, Alistarh-EF-NIPS2018}. While these early breakthroughs were followed by a steady stream of works offering various improvements and generalizations, the current theoretical understanding of \algname{EF} is still very limited. Indeed, to the best of our knowledge, all existing analyses either i) apply to the single node setting only, ii) rely on very strong and often unreasonable assumptions, such global boundedness of the gradients, or iterate-dependent assumptions that cannot be checked a-priori and may not hold in practice, or iii) circumvent these issues via the introduction of additional unbiased compressors, which increase the communication cost. In this work we fix all these deficiencies by proposing and analyzing a new \algname{EF} mechanism, which we call \algname{EF21}, which consistently and substantially outperforms \algname{EF} in practice. Moreover, our theoretical analysis relies on standard assumptions only, works in the distributed heterogeneous data  setting, and leads to better and more meaningful rates. In particular, we prove that \algname{EF21} enjoys a fast $O(1/T)$  convergence rate for smooth nonconvex problems, beating the previous bound of $O(1/T^{2/3})$, which was shown under  a strong bounded gradients assumption. We further improve this to a fast linear rate for Polyak-Lojasiewicz functions, which is the first linear convergence result for an error feedback method not relying on unbiased compressors. Since \algname{EF} has a large number of applications where it reigns supreme, we believe that our 2021 variant, \algname{EF21}, can a large impact on the practice of communication efficient distributed learning.	
	\end{abstract}

	      \tableofcontents
	\section{Introduction}

In order to obtain state-of-the-art performance, modern machine learning models rely on elaborate architectures, need to be trained on data sets of enormous sizes, and involve a very large number of parameters. Some of the most successful models are heavily over-parameterized, which means that they involve more parameters than the number of available training data points \citep{ACH-overparameterized-2018}.  Naturally, these circumstances should inform the design of optimization methods that could be most efficient  to perform the training. 

First, the reliance on sophisticated model architectures, as opposed to simple linear models, generally leads to {\bf nonconvex} optimization problems, which are more challenging than convex problems~\citep{NonconvexBook}.  Second, the need for very large training data sizes necessitates the use of {\bf distributed computing}~\citep{DMLsurvey}. Due to its enormous size, the data needs to be partitioned across a number of machines  able to work in parallel. Typically, for further efficiency gains, each such machine further parallelizes its local computations using one or more hardware accelerators. Third, the very large number of parameters describing these models exerts an extra stress on the communication links used to exchange model updates among the machines. These links are typically slow compared to the speed at which computation takes place, and communication often forms the bottleneck of distributed systems even in less extreme situations than over-parameterized training where the number of parameters, and hence the nominal size of communicated messages, can be truly staggering. For this reason, modern efficient optimization methods typically employ elaborate {\bf lossy communication compression} techniques to reduce the size of the communicated messages.

Due to the above reasons, in this paper we are interested in solving the {\em nonconvex distributed optimization problem}
	\begin{equation}\label{eq:finite_sum}
	\squeeze \min \limits_{x\in \R^d} \left[f(x) \eqdef \frac{1}{n} \sum \limits_{i=1}^n f_i(x)\right],
	\end{equation}
where $x\in \R^d$ represents the parameters of a machine learning model we wish to train, $n$ is the number of workers/nodes/machines, and $f_i(x)$ is the loss of model $x$ on the data stored on node~$i$. We specifically focus on the development of new and more efficient communication efficient first-order methods for solving \eqref{eq:finite_sum} utilizing {\em biased} compression operators, with a special emphasis on  {\em clean convergence analysis} which removes the {\em strong and often unrealistic assumptions}, such as the bounded gradient assumption, which are currently needed to analyze such methods (see Table~\ref{tbl:many_methods}). 

The remainder of the paper is organized as follows. In Section~\ref{sec:motivation} we describe the key concepts, results and open problems that form the motivation for our work, and summarize our main contributions. Our main theoretical results are presented in Section~\ref{sec:results}. In Section~\ref{sec:equivalence} we establish a connection between \algname{EF} and \algname{EF21}. Finally, experimental results are described in Section~\ref{sec:experiments}.

\section{Background and Motivation} \label{sec:motivation}

To better motivate our approach and contributions,   we first offer a concise walk-through over the key considerations, difficulties, advances and open problems in this area.

\subsection{Two families of compression operators}
Compression is typically performed via the application of a (possibly randomized) mapping $\cC:\R^d\to \R^d$, where $d$ is the dimension of the vector/tensor that needs to be communicated, with the property that it is much easier/quicker to transfer $\cC(x)$ than it is to transfer the original message $x$. This can be achieved in several ways, for instance by sparsifying the input vector \citep{Alistarh-EF-NIPS2018}, or by quantizing its entries \citep{alistarh2017qsgd, Cnat}, or via a combination of these and other approaches \citep{Cnat, beznosikov2020biased}.

There are two large classes of compression operators $\cC$ often studied in the literature:  i) {\bf unbiased compression operators} satisfying a variance bound proportional to the square norm of the input vector, and  ii) {\bf biased compression operators} whose square distortion is contractive with respect to the square norm of the input vector.  

In particular, we say that a (possibly randomized) map $ \cC: \R^{d} \rightarrow \R^{d}$ is an {\em unbiased compression operator}, or simply just {\em unbiased compressor}, if there exists a constant $\omega \geq 0$ such that
		\begin{eqnarray}\label{eq:u_compressor}
		\Exp{\cC(x)}=x, \quad \Exp{\|\cC(x) - x\|^{2}} \leq \omega\|x\|^{2}, \qquad \forall x\in \R^d.
		\end{eqnarray}
The family of such operators will be denoted by $\bU(\omega)$. Further, we say that a (possibly randomized) map $\cC: \R^{d} \rightarrow \R^{d}$ is a {\em biased compression operator}, or simply just {\em biased compressor}, if   there exists a constant $0<\alpha\leq 1$ such that		\begin{eqnarray}\label{eq:b_compressor}
		\Exp{\|\cC(x) - x\|^{2}} \leq \rb{1 - \alpha} \|x\|^{2}, \qquad \forall x\in \R^d.
		\end{eqnarray}
The family of such operators will be denoted by $\bB(\alpha)$. It is well known that, in a certain sense, the latter class contains the former. In particular, it is easy to verify that if $\cC\in \bU(\omega)$, then $(1+\omega)^{-1}\cC\in \bB(\nicefrac{1}{(1+\omega)})$. However, the latter class is strictly larger, i.e., it contains compressors which do not arise via a scaling of an unbiased compressor. A canonical example of this is the Top-$k$ compressor, which preserves the $k$ largest (in absolute value) entries of the input, and zeros out the remaining entries, and for which $\alpha=\nicefrac{k}{d}$. We refer to \citep[Table 1]{beznosikov2020biased,UP2021} for more examples of unbiased and biased compressors, and to \cite{Dutta-compress-Survey-2020} for a systems-oriented survey.

When used in an appropriate way, greedy biased compressors, such as Top-$k$, are often empirically superior to their unbiased counterparts \citep{Seide2014}, such as Rand-$k$. Intuitively, such greedy compressors retain more of the ``information'' or ``energy'' contained within the message, and hence introduce less distortion. This is beneficial in practice, at least in the simplistic single node (i.e., non-distributed) setting, albeit even here we do not have convincing theory that would explain this. 
\notneeded{Indeed, both Top-$k$ and Rand-$k$ impart the same distortion in the worst case, which happens when the energy is distributed uniformly across all entries of the input vector, and it is not easy to capture theoretically that this worst case situation will not happen repeatedly throughout the iterations. As a result, there is currently no separation between the worst case complexity of first order methods, such as gradient descent, combined with biased vs related unbiased compressors \citep{beznosikov2020biased}. If one makes a-priori statistical assumptions on the distribution of the messages/gradients that need to be compressed, such a separation can be made \citep{beznosikov2020biased}. While insightful, this is not satisfactory. Indeed, the gradients produced by methods such as gradient descent evolve in a non-stationary way, and hence modeling them as samples coming from a fixed distribution raises questions. Further, gradient compression affects the iterates and hence also the gradients that will be produced in all subsequent iterations, which is another phenomenon not captured by the aforementioned approach.}

\subsection{Error feedback: what it is good for, and what we still do not know}

The difference between what we know about unbiased and biased compressors is larger still in the distributed setting. 

In particular, {\em unbiasedness} turns out to be a very effective tool facilitating the analysis of distributed first order methods utilizing unbiased compressors, and for this reason, the landscape of methods using such compressors is very rich and relatively well understood. For example, using unbiased compressors we know how to 
\begin{itemize} 
\item [i)] analyze distributed compressed gradient decent~\citep{DCGD,99percent}, 
\item [ii)] remove the variance introduced by compression  to achieve faster convergence~\citep{DIANA,DIANA2,99percent}, 
\item [iii)] perform bidirectional compression at the workers and also at the master~\citep{Cnat,Artemis2020,Lin_EC_SGD}, 
\item [iv)] develop a general theory for SGD which, besides more standard methods, also includes variants using unbiased compression of (stochastic) gradients~\citep{sigma_k, sigma_k-convex, Nonconvex-sigma_k}, 
\item [v)] achieve Nesterov acceleration in the strongly convex regime~\citep{ADIANA}, 
\item [vi)] how to analyze these methods in the nonconvex regime~\citep{DIANA, Cnat, Nonconvex-sigma_k}, 
\item [vii)]  achieve acceleration in the nonconvex regime~\citep{MARINA}, and even how to 
\item [viii)] apply unbiased compressors to Hessian matrices to obtain communication-efficient second-order methods~\citep{NL2021}.
\end{itemize}

The situation with {\em general}  biased compressors (i.e., those that do not arise from unbiased compressors via scaling) is much more challenging. The key complication comes from the fact that their naive use within first order methods, such as gradient descent, can lead to divergence. We refer the reader to  \citep[Example 1]{beznosikov2020biased} for a simple example where gradient descent ``enhanced'' with the Top-1 compressor leads to {\em exponential divergence} when applied to the problem of minimizing the average of three strongly convex quadratics in $\R^3$. However, divergence of gradient descent enhanced with biased compressors such as Top-$k$ was observed empirically much sooner, and a fix for this problem, known as {\em error feedback} (\algname{EF}), or {\em error compensation} (\algname{EC}), was suggested by~\citet{Seide2014}. This fix remained a heuristic until very recently. 

The first theoretical breakthroughs focused on the simpler single-node setting \citep{Stich-EF-NIPS2018, Alistarh-EF-NIPS2018}.  The first analysis in the general distributed heterogeneous data\footnote{Problem~\eqref{eq:finite_sum} is in the {\em heterogeneous data regime} if no similarity among the functions (and hence among the data stored across different nodes giving rise to these functions) is assumed. 
% A typical similarity assumption  would be to require that  the inequality $\norm{\nabla f_i(x) - \nabla f_j(x)} \leq C$ holds for all $x\in \R^d$ and all $i,j\in \{1,2,\dots,n\}$, and some constant $C>0$. This assumption is too strong, and does not even hold for convex quadratics.
}
setting was performed by \citet{beznosikov2020biased}, and was confined to the strongly convex regime. While without compression, one can expect a linear rate, the rate in \citep{beznosikov2020biased} is linear only in the special case of an over-parameterized regime (i.e., regime in which the loss functions on all nodes share a common minimizer) with a requirement of full gradient computations on each node. These deficiencies were later fixed by \citet{Lin_EC_SGD}, who developed the first linearly convergent methods \algname{EC-GD-DIANA} and \algname{EC-LSVRG-DIANA}, and also analyzed the convex case. 
\notneeded{Further, \citet{EC-Katyusha} showed that error-compensated methods can be {\em accelerated} in the sense of Nesterov~\citep{NesterovBook}. }
However, this advance was achieved through the use of additional unbiased compressors, and hence via an increase in communication in each round. 

\begin{quote}{\em In particular, whether it is possible to obtain a linearly convergent error-compensated method in the general heterogeneous data setting, relying on biased compressors only, is still an open problem. }
\end{quote}

The current state-of-the-art theoretical result for error-compensated methods in the smooth non-convex regime are due to \citet[Theorem 4.1]{Koloskova2019DecentralizedDL}, who consider the more general problem of decentralized optimization over a network. In the case when full (as opposed to stochastic) gradients are computed on each node, they  show that after $T$ communication rounds it is possible to find a random vector $\hat{x}^T$ with the guarantee \begin{equation}\label{eq:T^{2/3}} \squeeze \Exp{\sqnorm{\nabla f(\hat{x}^T)}} =\cO\left(\frac{G^{2/3}}{T^{2/3}}\right),\end{equation} under the {\em bounded gradient} assumption which requires the existence of a constant $G>0$  such that
\begin{equation}\label{eq:G}\sqnorm{\nabla f_i(x)} \leq G^2\end{equation}
holds for all $x\in \R^d$ and all $i\in \{1,2,\dots,n\}$. This was a slight improvement in rate over an result obtained by \citet{Lian2017}, who instead use the bounded dissimilarity assumption
\begin{equation}\label{eq:bounded-diss} \squeeze \frac{1}{n}\sum \limits_{i=1}^n \sqnorm{\nabla f_i(x) - \nabla f(x)} \leq G^2. \end{equation}
A summary of the limitations of known results for \algname{EF}-based methods is provided in Table~\ref{tbl:many_methods}.

	\begin{table}[t]
		\caption{Known results for first order methods using biased compressors. sCVX = supports strongly convex functions, nCVX= supports nonconvex functions, DIST = works in the distributed regime. ${}^\dagger$decentralized method}\label{tbl:many_methods}
		\centering		
				 \footnotesize
		\begin{tabular}{ccccc}
			%    \hline 
			\bf Algorithm &  \bf sCVX &  \bf nCVX &  \bf DIST &  \bf key limitation \\
%%%%%%%%%%%%%%%%%%%%%%%%%%%		
			\hline 
			\begin{tabular}{c} \algname{EF} \\ \cite{Stich-EF-NIPS2018} \end{tabular}& \cmark & \xmark & \cmark & \begin{tabular}{c}bounded gradients; \\ sublinear rate in sCVX case\end{tabular} \\
%%%%%%%%%%%%%%%%%%%%%%%%%%%								
			\hline
			 \begin{tabular}{c}\algname{EF-SGD} \\\cite{Stich2019TheEF} \end{tabular}& \cmark & \cmark & \xmark & single node only \\ 	
%%%%%%%%%%%%%%%%%%%%%%%%%%%	
			\hline \begin{tabular}{c} \algname{EF} \\ \cite{stich2020biased} \end{tabular} & \cmark & \cmark & \xmark & single node only \\ 
%%%%%%%%%%%%%%%%%%%%%%%%%%%			
			\hline
			\begin{tabular}{c}\algname{SignSGD}\\ \cite{Karimireddy_SignSGD} \end{tabular} & \xmark & \cmark & \xmark & \begin{tabular}{c} moment  bound; \\ single node only \end{tabular} \\	
%%%%%%%%%%%%%%%%%%%%%%%%%%%			
			\hline \begin{tabular}{c} \algname{EC-SGD} \\ \cite{beznosikov2020biased} \end{tabular}& \cmark & \xmark & \cmark & \begin{tabular}{c}linear rate only \\ if $\nabla f_i(x^\star)=0\; \forall i$ \end{tabular} \\ 
					 		
%%%%%%%%%%%%%%%%%%%%%%%%%%%				
			\hline \begin{tabular}{c} \algname{EC-SGD} \\ \cite{Lin_EC_SGD} \end{tabular} & \cmark & \xmark & \cmark & \begin{tabular}{c} linear rate only using  \\  an extra unbiased compressor \end{tabular}\\ 
%%%%%%%%%%%%%%%%%%%%%%%%%%%							 
			 \hline
			\begin{tabular}{c}\algname{DoubleSqueeze} \\ \cite{DoubleSqueeze} \end{tabular}& \xmark & \cmark & \cmark & \begin{tabular}{c}bounded compression error; \\ slow $O(\nicefrac{1}{T^{2/3}})$ rate in nCVX case \end{tabular} \\			
%%%%%%%%%%%%%%%%%%%%%%%%%%%						
			\hline  \begin{tabular}{c} \algname{Qsparse-SGD}, \algname{CSER}  \\\cite{Qsparse_local_SGD,CSER} \end{tabular}&   \cmark & \cmark & \cmark & \begin{tabular}{c}bounded gradients; \\ slow $O(\nicefrac{1}{T^{1/2}})$ rate in nCVX case \end{tabular} \\
%%%%%%%%%%%%%%%%%%%%%%%%%%%			
			\hline \begin{tabular}{c} \algname{EC-SGD} \\\cite{Koloskova2019DecentralizedDL} \end{tabular} & \xmark & \cmark & \cmark$^\dagger$ & \begin{tabular}{c}bounded gradients; \\ slow $O(\nicefrac{1}{T^{2/3}})$ rate in nCVX case\end{tabular} \\
%%%%%%%%%%%%%%%%%%%%%%%%%%%			
			\hline        
		\end{tabular}      		   
	\end{table}

\begin{quote}{\em In this work we argue that  the bounded gradients \eqref{eq:G} and bounded dissimilarity \eqref{eq:bounded-diss} assumptions are too strong\footnote{The bounded gradient \eqref{eq:G} and bounded dissimilarity \eqref{eq:bounded-diss} assumptions are too strong as they are rarely satisfied. For example, neither hold even for simple quadratic functions. To see this, let $f_i(x)=x^\top \bA_i x$, where $\bA_i\in \R^{d\times d}$. Since $\nabla f_i(x)= \mB_i x$, where $\mB_i=\bA_i + \bA_i^\top$, the bounded gradient assumption requires the vectors $\sup_{x} \max_{i}\norm{\mB_i x}$ to be bounded, which is not the case, unless all matrices $\mB_i$ are zero. The bounded dissimilarity assumption \eqref{eq:bounded-diss}, which can be written in the form $\frac{1}{n} \sum_{i=1}^n  \| (\mB_i - \frac{1}{n}\sum_{j=1}^n \mB_j )x \|^2 \leq G^2,$
also does not hold, unless $\mB_i = \mB_j$ for all $i,j$, which reduces to the identical data regime, which is of limited interest.}, and that the sublinear rate \eqref{eq:T^{2/3}} is not what one should expect from a good analysis of a well designed error-compensated first-order method. Instead, one could hope for the faster $\cO(\nicefrac{1}{T})$ rate, which is what one obtains with methods using unbiased compressors \cite[Theorem 2.1]{MARINA}. The resolution of these issues is an open problem.  }
\end{quote}

\subsection{Summary of contributions} \label{sec:contributions}

In this work we address and resolve the aforementioned challenges. Our key contributions are:

\begin{itemize}
\item [] {\bf A. New error feedback mechanism.} We propose a new error feedback (resp.\ error compensation) mechanism, which we call \algname{EF21} (resp.\ \algname{EC21}) -- see Algorithms~\ref{alg:EF21-singlenode} and~\ref{alg:EF21}.  Unlike most results on error compensation, \algname{EF21} naturally works in the distributed heterogeneous data setting.

\begin{table}[]
\caption{Summary of complexity results obtained in this paper. Quantities: $\mu$ = PL constant; $\gamma$ = stepsize; $G^0$ = see \eqref{eq:G^t}; $\Psi^t =$ Lyapunov function defined in Theorem~\ref{thm:PL-main}.  }
\label{tbl:complexity}
% \footnotesize
\centering
\begin{tabular}{|c|c|c|c|}
\hline
 \bf Assumptions & \bf Complexity & \bf Theorem \\
\hline
 \begin{tabular}{c}$f_i$ is $L_i$-smooth \\ $f$ is lower bounded by $\finf$ \end{tabular}  &$\Exp{\sqnorm{\nabla f(\hat{x}^{T})} } \leq \frac{2\left(f(x^{0})-f^{\text {inf }}\right)}{\gamma T} + \frac{\Exp{G^0}}{\theta T}$ & \ref{thm:main-distrib} \\
\hline
  \begin{tabular}{c}$f_i$ is $L_i$-smooth \\ $f$ is lower bounded by $\finf$ \\ $f$ satisfies PL condition \end{tabular}  & $\Exp{ \Psi^T } \leq (1- \gamma \mu )^T \Exp{\Psi^0}$ & \ref{thm:PL-main} \\
\hline
\end{tabular}
\end{table}

\item []
 {\bf B. Standard assumptions and fast rates.} Our theoretical analysis of \algname{EF21} relies on standard assumptions only, which are: \begin{itemize}
 \item [i)] $L_i$-smoothness of the individual functions $f_i$, and
 \item [ii)] existence of a global lower bound $\finf \in \R$ on $f$. 
 \end{itemize}
 We prove that under these assumptions, \algname{EF21} enjoys the desirable $\cO(\nicefrac{1}{T})$ convergence rate, which improves upon the previous $\cO(\nicefrac{1}{T^{2/3}})$ state-of-the-art result of \citet{Koloskova2019DecentralizedDL} both in terms of the rate, and in terms of the strength of the  assumptions needed to obtain this result.   These complexity results are summarized in the first row of Table~\ref{tbl:complexity}.

\item []
{\bf  C. Linear rate for Polyak-Lojasiewicz functions.} We show that under the additional assumption that $f$ satisfies the Polyak-Lojasiewicz inequality, \algname{EF21} enjoys a linear convergence rate.  This improves upon the results of \citet{beznosikov2020biased}, who only obtain a linear rate in the case when $\nabla f_i(x^\star)=0$ for all $i$, where $x^\star=\arg\min f$, and provides an alternative to the linear convergence results of \citet{Lin_EC_SGD}, who needed to introduce additional unbiased compressors into their scheme, and hence additional communication, in order to obtain their results. 
% Our rates are  at the same time better than current state of the art, apply to a much wider class of problems, and are also more meaningful because the classical \algname{EF} mechanism is often applied to problems that do not satisfy  bounded  gradient or bounded dissimilarity assumptions.  
Our complexity results are summarized in  second row of Table~\ref{tbl:complexity}.

 \item []
{\bf D. Empirical superiority.} We show through extensive numerical experimentation on  both synthetic problems and deep learning benchmarks that \algname{EF21}  consistently and substantially outperforms \algname{EF} in practice. One of the reasons behind this  is the fact that our method is able to admit much larger learning rates. Since \algname{EF} has a large number of applications where it reigns supreme, we believe that \algname{EF21} will have a large impact on the practice of communication efficient distributed learning. 

\item []
{\bf E. A more aggressive variant.} We further propose a more aggressive variant, \algname{EF21+} (see Section~\ref{sec:E21+}), which has an even better empirical behavior. We show that if $\cC$ is deterministic, the same theorems capturing the convergence of \algname{EF} hold for \algname{EF21} as well.

\item []
 {\bf F. Stochastic setting.} We describe an extension to the stochastic setting, i.e., when each node computes a stochastic gradient instead of the exact/full gradient,  in Appendix~\ref{sec:extensions}.

\end{itemize}

\section{Main Results} \label{sec:results}

Since we are about to re-engineer the classical error feedback technique, it will be useful to take a step back and re-examine the issues inherent to the simplest first order method which uses biased compressors but  does {\em not} employ error feedback: distributed compressed gradient descent (\algname{DCGD}). 

Let $x^t$ be the $t$-th iterate, shared by all $n$ nodes. Each node $i$ first computes its local gradient $\nabla f_i(x^t)$, compresses is using some $\cC \in \bB(\alpha)$, and sends the compressed gradient $\cC(\nabla f_i(x^t))$ to the master. The master aggregates all $n$ messages via averaging, and performs the optimization step
\begin{equation} \squeeze x^{t+1} = x^t - \frac{\gamma}{n} \sum \limits_{i=1}^n \cC(\nabla f_i(x^t)).\label{eq:DCGD}\end{equation}
As mentioned before, this method can diverge, even in simple quadratic problems in low dimensions~\citep{beznosikov2020biased}. Let us look at this problem from a different angle. Assume, for the sake of an intuitive argument, that the sequence of iterates actually converges to some $x^\dagger$. Since in general there is no reason for the gradients $\nabla f_i(x^\dagger)$ to be all zero, even if $x^\dagger$ is the minimizer of $f$,  the application of $\cC$ to the gradients $\nabla f_i(x^t)$ will introduce a nonzero distortion even if 
$x^t\approx x^\dagger$. Indeed, in view of \eqref{eq:b_compressor}, all that can be guaranteed is that $$\Exp{\sqnorm{\cC(\nabla f_i(x^t))-\nabla f_i(x^t)}} \leq (1-\alpha) \sqnorm{\nabla f_i(x^t)},$$ which can be large if the norm of $\nabla f_i(x^t)$ is large. So, the method is intrinsically unstable around $x^\dagger$, and hence can not converge to $x^\dagger$.

Our idea is to fix this issue by {\em compressing different vectors} instead of the gradients, vectors that would hopefully converge to zeros instead. Since in view of \eqref{eq:b_compressor} the application  of $\cC$ to progressively vanishing vectors introduces progressively vanishing distortion, the stabilization problem would be solved. But what vectors should we compress? In order to answer this question, it will be useful to consider a simpler and more abstract setting first, which we shall do next.

\subsection{Markov compressors}\label{sec:Markov}

Assume we are given a sequence of input vectors $\{v^t\}_{t\geq 0}$  (e.g., gradients) generated by some algorithm. This sequence does not necessarily converge to zero. Our goal is to produce a sequence of ``good'' and ``easy to communicate'' (to some entity, which we shall call the ``master'') estimates of these vectors, making use of a compressor $\cC\in \bB(\alpha)$.  Let us proceed through several steps of discovery.

\paragraph{Naive idea.} The first and naive approach, described above, is to simply output the sequence of compressed inputs: $\{\cC(v^t)\}_{t\geq 0}$. However, while these estimates can be communicated efficiently, they are not getting ``better''. That is, the distortion $\Exp{\sqnorm{\cC(v^t) -v^t}}$ is not necessarily improving.
 
\paragraph{Good but not implementable idea.} What can we do better? Consider the following idea. If we new, hypothetically, the limit of this sequence, $v^*$, we could  output  $v^* + \cC(v^t - v^*)$ at iteration $t$ instead. Since $v^t\to v^*$, the distortion between the input and the output at iteration $t$ is
\[\Exp{\sqnorm{v^* + \cC(v^t - v^*) - v^t} } = \Exp{\sqnorm{ \cC(v^t - v^*) - (v^t-v^*)} } \overset{\eqref{eq:b_compressor}}{\leq}  (1-\alpha)\sqnorm{v^t-v^*} \to 0. \]
So, the distortion issue is fixed! Moreover, if we assume the master knows $v^*$, then the output vector at each iteration can be communicated cheaply as well, since all we need to communicate is the compressed vector $\cC(v^t - v^*)$. It will be useful to think of this operation as a new compressor, called $\cC_{v^*}$, one that takes $v^t$ as an input, and gives $v^* + \cC(v^t - v^*)$ as its output. That is, we can define \begin{equation}\label{eq:C-star}\cC_{v^*}(v) \eqdef v^* + \cC(v - v^*).\end{equation}
While the compressor $\cC_{v^*}$ satisfies all our requirements, it is not implementable, since the vector $v^*$ is not known. We will now use this intuition to construct an implementable mechanism.

\paragraph{Good and implementable idea.} In the above construction, we have used the fact that $v^t- v^* \to 0$ to construct a good mechanism, but one that is not implementable. How can we fix this issue? The rescue comes from the {\em recursive} observation that if we indeed succeed in constructing a compressor, let's call it $\cM$, such that the distortion between $\cM(v^t)$ and $v^t$ vanishes as $t\to \infty$, then it must be the case that $v^t - \cM(v^t) \to 0$. So, we can compress {\em this} vanishing vector instead. This idea gives rise to the following recursive definition of $\cM$: 
		\begin{eqnarray}  \cM(v^0) &\eqdef& \cC(v^0) \label{eq:cM-init} \\ \cM(v^{t+1}) &\eqdef & \cM(v^t) + \cC(v^{t+1} - \cM(v^t)), \quad t\geq 0 \label{eq:cM-recursion}\end{eqnarray}
Note that \eqref{eq:cM-recursion} is similar to \eqref{eq:C-star}, with one key difference: we are using the previously compressed vector $\cM(v^t)$ instead of the limit vector $v^*$. This property also makes our new compressor non-stationary, i.e., it has a Markov property. 

It is easy to establish that (see Appendix~\ref{sec:markov-proofs}) under some assumptions about the speed at which the input sequence $v^t$ converges to $v^*$, it will be the case that $\Exp{\sqnorm{\cM(v^t) - v^t}} \to 0$. For instance, if the convergence rate of the input sequence is linear, then the distortion will converge to  0. While this is interesting on its own, let us deploy our new tool, which we call {\em Markov compressor}, in the context of gradient descent, and then in the context of distributed gradient descent.

\subsection{Compressed gradient descent using the Markov compressor}

			\begin{algorithm}[t]
		\centering
		\caption{\algname{EF21} (Single node)}\label{alg:EF21-singlenode}
		\begin{algorithmic}[1]
			\State \textbf{Input:} starting point $x^{0} \in \R^d$, learning rate $\gamma>0$, $g^0 = \cC(\nabla f(x^0))$ 
			\For{$t=0,1, 2, \dots, T-1 $}
			\State $x^{t+1} = x^t - \gamma  g^t$ 
			\State $g^{t+1} = g^t + \cC( \nabla f(x^{t+1}) - g^t)$
			\EndFor
		\end{algorithmic}
	\end{algorithm}
	
For simplicity, consider solving  problem \eqref{eq:finite_sum} in $n=1$ case, i.e., the problem \begin{equation}\label{eq:problem-1-node}\min_{x\in\R^d} f(x),\end{equation} using the compressed gradient descent method featuring the Markov compressor. Start with $x^0\in \R^d$, stepsize $\gamma>0$, and let $$\cM(\nabla f(x^0) ) = \cC(\nabla f(x^0) ).$$ After this, for $t\geq 0$ iterate:
\begin{eqnarray} x^{t+1} & = & x^t - \gamma \cM( \nabla f(x^t) ) \label{eq:89f9dh8fd} \\
\cM( \nabla f(x^{t+1}) ) & = & \cM(\nabla f(x^t)) + \cC(\nabla f(x^{t+1}) - \cM(\nabla f(x^t))). \label{eq: uyd8s7tg9fid}
\end{eqnarray}

Note that the situation here is more complicated than the abstract setting described earlier since now there is {\em interaction} between the input sequence $\{\nabla f(x^t)\}_{t\geq 0}$ of gradients  and the sequence $\cM(\nabla f(x^t))$ of compressed gradients via the Markov compressor. Indeed, the output of $\cM$ at iteration $t$ influences the next iterate $x^{t+1}$ (via \eqref{eq:89f9dh8fd}), which in turn defines the next input vector $v^{t+1} = \nabla f(x^{t+1})$ in the sequence, and so on.

To lighten up the heavy notation in \eqref{eq:89f9dh8fd} and \eqref{eq: uyd8s7tg9fid}, it will be useful to write $g^t  = \cM(\nabla f(x^t))$. Using this new notation that hides the fact that $g^t$ is the application of the Markov compressor to the gradient, the method described above is formalized as Algorithm~\eqref{alg:EF21-singlenode}. This is precisely our proposed new variant of error feedback, \algname{EF21}, specialized to the  single node problem \eqref{eq:problem-1-node}.

\subsection{Distributed variant of \algname{EF21}} 

The main method of this paper, which we now present as Algorithm~\ref{alg:EF21}, is an extension  of Algorithm~\ref{alg:EF21-singlenode} to the general finite-sum problem \eqref{eq:finite_sum}. In particular, we apply the Markov compressor individually on each node to the local gradients $\nabla f_i(x^t)$, and communicate the compressed gradients to the master. Recall that we only need to communicate the vectors  $\cC(\nabla f_i(x^{t+1}) - g_i^t)$ since the additive terms $g_i^t$ appearing in the Markov compressor were communicated in the previous round. Master then averages all gradient estimators, obtaining $g^{t+1} = \frac{1}{n} \sum_{i=1}^n  g_i^{t+1}$, which can be done by performing the calculation $g^{t+1} = g^t + \frac{1}{n} \sum_{i=1}^n c_i^t $, where $g^t$ is the average from the previous round which the master maintains, and 
$c_i^t = \cC(\nabla f_i(x^{t+1}) - g_i^t)$ are the compressed messages. After this, the master takes a gradient-like step, and broadcasts the new model to all nodes.

				\begin{algorithm}[t]
		\centering
		\caption{\algname{EF21} (Multiple nodes)}\label{alg:EF21}
		\begin{algorithmic}[1]
			\State \textbf{Input:} starting point $x^{0} \in \R^d$;  $g_i^0 = \cC(\nabla f_i(x^0))$ for $i=1,\dots, n$ (known by nodes and the master); learning rate $\gamma>0$; $g^0 = \frac{1}{n}\sum_{i=1}^n g_i^0$ (known by master)
			\For{$t=0,1, 2, \dots , T-1 $}
			\State Master computes $x^{t+1} = x^t - \gamma g^t$ and broadcasts $x^{t+1}$ to all nodes
			\For{{\bf all nodes $i =1,\dots, n$ in parallel}}
			\State Compress $c_i^t = \cC(\nabla f_i(x^{t+1}) - g_i^t)$ and send $c_i^t $ to the master
			\State Update local state $g_i^{t+1} = g_i^t + \cC( \nabla f_i(x^{t+1}) - g_i^t)$
			\EndFor
			\State Master computes $g^{t+1} = \frac{1}{n} \sum_{i=1}^n  g_i^{t+1}$ via  $g^{t+1} = g^t + \frac{1}{n} \sum_{i=1}^n c_i^t $
			\EndFor
		\end{algorithmic}
	\end{algorithm}

	%%%%%%%%%%%%%%%%%%%%%%%%
%%%%%%%%%%%%%%%%%%%%%%%%
\subsection{Theory}\label{sec:theory}
%%%%%%%%%%%%%%%%%%%%%%%%
%%%%%%%%%%%%%%%%%%%%%%%%

	We make the following assumption throughout:	
	\begin{assumption}[Smoothness and lower boundedness]\label{as:main}	 
Every $f_i$ has $L_i$-Lipschitz gradient, i.e., $$\norm{\nabla f_i(x) - \nabla f_i(y)} \le L_i\norm{x - y}$$ for all $x, y\in \R^d$, and $\finf \eqdef \inf_{x\in \R^d} f(x)>-\infty $. 	\end{assumption}

If each $f_i$ has $L_i$-Lipschitz gradient, then it is straightforward to check by Jensen's inequality that $f$ is $L$-Lipschitz, with $L$ satisfying the inequality $L\leq \frac{1}{n}\sum_i L_i$. It will be also useful to define 
$\wL \eqdef (\frac{1}{n}\sum_{i=1}^n L_i^2)^{1/2}.$ By the arithmetic-quadratic mean inequality, we have $\frac{1}{n}\sum_i L_i \leq \wL$.	Let  \begin{equation}\label{eq:G^t} \squeeze G^t \eqdef \frac{1}{n} \sum \limits_{i=1}^n \sqnorm{ g_i^t-\nabla f_i(x^{t}) },\end{equation}	
a quantity which will appear in both our theorems.  In \algname{EF21} we use $\cC \in \bB(\alpha)$, where $0<\alpha\leq 1$, and define $\theta = 1-\sqrt{1-\alpha}$ and  $\beta = \frac{1-\alpha}{1-\sqrt{1-\alpha}}$. We now formulate our first complexity result.

	\begin{theorem}\label{thm:main-distrib}
		Let Assumption~\ref{as:main} hold, and let the stepsize in Algorithm~\ref{alg:EF21} be set as
\begin{equation} \label{eq:manin-nonconvex-stepsize}
		0<\gamma \leq \rb{L + \wL\sqrt{\frac{\beta}{\theta}}}^{-1}.
\end{equation}
		Fix $T \geq 1$ and let $\hat{x}^{T}$ be chosen from the iterates $x^{0}, x^{1}, \ldots, x^{T-1}$  uniformly at random. Then
\begin{equation} \label{eq:main-nonconvex}		
\squeeze \Exp{\sqnorm{\nabla f(\hat{x}^{T})} } \leq \frac{2\left(f(x^{0})-f^{\text {inf }}\right)}{\gamma T} + \frac{\Exp{G^0}}{\theta T} . 
\end{equation}		
	\end{theorem}
	
Note that $\sqrt{\nicefrac{\beta}{\theta}}=\frac{(1+\sqrt{1-\alpha})}{\alpha} -1 \leq \frac{2}{\alpha}-1$ is decreasing in $\alpha$. This makes sense since larger $\alpha$ means less dramatic compression, which leads to smaller $\sqrt{\nicefrac{\beta}{\theta}}$, and this through \eqref{eq:manin-nonconvex-stepsize} allows for larger stepsize, and hence fewer communication rounds.
	We now introduce the PL assumption, which enables us to obtain a linear convergence result.
	\begin{assumption}[Polyak-Lojasiewicz] \label{ass:PL} There exists $\mu>0$ such that $f(x) - f(x^\star) \leq  \frac{1}{2 \mu}\sqnorm{\nabla f(x)}$ for all $x\in \R^d$, where $x^\star = \arg\min f$.
	\end{assumption}
	
	\begin{theorem} \label{thm:PL-main} Let Assumptions~\ref{as:main} and~\ref{ass:PL} hold, and 	 let the stepsize in Algorithm~\ref{alg:EF21} be set as 
\begin{equation} \label{eq:PL-stepsize}
\squeeze 		0<\gamma \leq \min \left\{ \rb{L + \wL\sqrt{\frac{2 \beta}{\theta}}}^{-1} , \frac{\theta}{2 \mu}\right\}.
\end{equation}
		Let $\Psi^t\eqdef f(x^{t})-f(x^\star)+ \frac{\gamma}{\theta}G^t$. Then for any $T\geq 0$, we have
\begin{equation} \label{PL-main}
		\Exp{ \Psi^T } \leq (1- \gamma \mu )^T \Exp{\Psi^0}.
\end{equation}
	\end{theorem}
	
Our theorems hold for an arbitrary choice of the initial vectors $\{g_i^0\}$, and not just for $g_i^0 = \cC(\nabla f_i(x^0))$. For instance, if $g_i^0=\nabla f_i(x^0)$ is used, then $\Exp{G^0}=0$, and the second term in \eqref{eq:main-nonconvex} vanishes.

\subsection{\algname{EF21+}: Use $\cC$ or the Markov compressor, whichever is better} 	\label{sec:E21+}

We now briefly describe a new hybrid method, called \algname{EF21+} (Algorithm~\ref{alg:EF21+}), which often performs particularly well in practice. We also show that both Theorems~\ref{thm:main-distrib} and \ref{thm:PL-main} still apply.
	
\paragraph{The \algname{EF21+} Algorithm.}	 In every communication round, \algname{EF21+}  allows each node  to compress using the ``best'' of $\cC$ and  the Markov compressor generated. So, \algname{EF21+} can be thought of as a hybrid between \algname{DCGD} (see \eqref{eq:DCGD}) and \algname{EF21}. The decision about which compressor to use is made by each node $i$ individually, based on which of the distortions $\norm{\cC(s) - s}$ and $\norm{\cM(s) - s}$ is smaller, where $s=\nabla f_i(x^{t+1})$. 

The method is formally described formally as Algorithm~\ref{alg:EF21+}.
	
\begin{algorithm}[h]
		\centering
		\caption{\algname{EF21+} (Multiple nodes)}\label{alg:EF21+}
		\begin{algorithmic}[1]
			\State \textbf{Input:} starting point $x^{0} \in \R^d$;  $g_i^0 = \cC(\nabla f_i(x^0))$ for $i=1,\dots, n$ (known by nodes and the master); learning rate $\gamma>0$; $g^0 = \frac{1}{n}\sum_{i=1}^n g_i^0$ (known by master)
			\For{$t=0,1, 2, \dots , T-1 $}
			\State Master computes $x^{t+1} = x^t - \gamma g^t$ and broadcasts $x^{t+1}$ to all nodes
			\For{{\bf all nodes $i =1,\dots, n$ in parallel}}
		        \State Compute gradient compressed by biased compressor $b_i^{t+1} = \cC(\nabla f_i(x^{t+1}))$
		         \State Compute gradient compressed my Markov compressor $m_i^{t+1} = g_i^t + \cC(\nabla f_i(x^{t+1}) - g_i^t)$
		         \State Compute distortions: $B_i^{t+1} = \sqnorm{b_i^{t+1} - \nabla f_i(x^{t+1})}$; $M_i^{t+1} = \sqnorm{m_i^{t+1} - \nabla f_i(x^{t+1})}$
		         \State Set $g_i^{t+1} = \begin{cases}m_i^{t+1} & \text{if} \quad M_i^{t+1} \leq B_i^{t+1} \\  b_i^{t+1} &  \text{if} \quad   M_i^{t+1} > B_i^{t+1}\end{cases}$
		         \EndFor
			\State Master computes $g^{t+1} = \frac{1}{n} \sum_{i=1}^n  g_i^{t+1}$ 
			\EndFor
		\end{algorithmic}
	\end{algorithm}

		\paragraph{Analysis of \algname{EF21+}.}
It is easy to see that both Theorem~\ref{thm:main-distrib} and Theorem~\ref{thm:PL-main} apply for \algname{EF21+} as well, under the additional assumption that $\cC$ is deterministic, such as Top-$k$. Here we only outline the proof. Note that the properties of $\cC$ appear in the proofs only through Lemma~\ref{lem:theta-beta} (see the appendix), which in the language of Algorithm~\ref{alg:EF21+} says that
\[\Exp{ M_i^{t+1} \;|\; W^t} \leq  (1-\theta)   G_i^t + \beta   \sqnorm{\nabla f_i(x^{t+1}) - \nabla f_i(x^t)} ,\]	
where $G_i^{t} = \sqnorm{g_i^t - \nabla f_i(x^t)}$. On the other hand, due	 to Step 8 in Algorithm~\ref{alg:EF21+}, we know that
\[ G_i^{t+1} \leq \min\{B_i^{t+1}, M_i^{t+1}\} \leq M_i^{t+1}.\]
Now, due to to the assumption that $\cC$ is a deterministic compressor, we have $\Exp{G_i^{t+1} \mid W^t} \leq G_i^{t+1}$. By stringing these three inequalities together, we arrive at
\[\Exp{ G_i^{t+1} \;|\; W^t} \leq  (1-\theta)   G_i^t + \beta   \sqnorm{\nabla f_i(x^{t+1}) - \nabla f_i(x^t)} ,\]	
and this inequality can be used in the proofs instead. The rest of the proof is identical.

\subsection{Dealing with stochastic gradients}\label{sec:stoch_grad-Main}

In Section~\ref{sec:extensions} (see Algorithm~\ref{alg:EF21-stoch}) we describe a natural extension of \algname{EF21} to the setting where full gradient computations are replaced by stochastic gradient estimators, i.e., we use a random vector
$$\hat{g}_i^t \approx	\nabla f_i(x^t).$$
We also outline how convergence analysis is performed in this regime.
	
				%%%%%%%%%%%%%%%%%%%%%%%%
	%%%%%%%%%%%%%%%%%%%%%%%%
\section{Relationship between \algname{\normalsize EF} and \algname{\normalsize EF21}} \label{sec:equivalence}
			%%%%%%%%%%%%%%%%%%%%%%%%
	%%%%%%%%%%%%%%%%%%%%%%%%
		
While this is not at all apparent at first sight, it turns out that \algname{EF} and \algname{EF21} are closely related. Before we describe this connection, however, we will first review the original \algname{EF} method.

\subsection{The original error feedback method}

The \algname{EF} method is described in Algorithm~\ref{alg:EF-original}. We write it in a slightly non-conventional but equivalent form which facilitates comparison with \algname{EF21}.

\begin{algorithm}[h]
			\centering
			\caption{\algname{EF} (Original error feedback)}\label{alg:EF-original}
			\begin{algorithmic}[1]
                  		\State Each node  $i=1,\dots, n$ sets the initial error to zero: {\color{green}$e^0_i = 0$} 
				\State Each node $i=1,\dots, n$ computes ${\color{blue}w_i^0} = \cC ( {\color{red}\gamma \nabla f_i(x^0)})$ and sends this to the master

				\For{$t=0, 1, 2, \dots, T-1 $}
				\State Master computes $x^{t+1} = x^t - \suminn {\color{blue}w_i^t}$
				\For{{\bf all nodes $i =1,\dots, n$ in parallel}}
				\State  Compute current error: ${\color{green}e_i^{t+1}} = {\color{green}e_i^t} + {\color{red}\gamma  \nabla f_i(x^t)} - {\color{blue}w_i^t} $
				\State Compute new local gradient $\nabla f_i(x^{t+1})$
				\State Compute error-compensated (stepsize-scaled) gradient ${\color{blue}w_i^{t+1}} = \cC({\color{green}e^{t+1}_i} + {\color{red}\gamma \nabla f_i(x^{t+1}))}$
				\State Send {\color{blue}$w_i^{t+1}$} to the master
				\EndFor
				\EndFor
			\end{algorithmic}
		\end{algorithm}	

 \algname{EF} works as follows. In iteration $t=0$, each node $i$ computes its local gradient $\nabla f_i(x^0)$, and ``would like'' to communicate the vector ${\color{red} \gamma \nabla f_i(x^0)}$ to the master, which is supposed to perform an aggregation of these vectors via averaging, and perform the gradient-type step $$x^1 = x^0 - \frac{1}{n} \sum_{i=1}^n {\color{red} \gamma \nabla f_i(x^0)}.$$ 
 This, in fact, is one step of gradient descent. However, the vector ${\color{red} \gamma \nabla f_i(x^0)}$ is hard to communicate. For this reason, this vector needs to be compressed, and the compressed version needs to be communicated instead. This would lead to the iteration
 $$x^1 = x^0 - \frac{1}{n} \sum_{i=1}^n {\color{blue} w_i^0}, \qquad \text{where} \qquad {\color{blue} w_i^0} = \cC( {\color{red}\gamma \nabla f_i(x^0)}),$$
which is a variant\footnote{This method {\em is}  \algname{DCGD} if $\cC$ is positively homogeneous, i.e., of $\cC(\gamma g) = \gamma \cC(g)$ for every $\gamma>0$ and $g\in \R^d$. However, even without positive homogeneity, this variant has the same theoretical properties as standard \algname{DCGD}.} of distributed \algname{CGD} (\algname{DCGD}). 

However, it is well known that \algname{DCGD}  may diverge. The key idea of error feedback is to compute the {\em \color{green}error}  $${\color{green}e_i^{1}} = {\color{red}\gamma  \nabla f_i(x^0)} -  \cC( {\color{red}\gamma \nabla f_i(x^0)}) =  {\color{red}\gamma  \nabla f_i(x^0)} - {\color{blue}w_i^0} ,$$
which is the difference between  the {\color{red} message $\gamma  \nabla f_i(x^0)$ we {\em want} to communicate}, and  the {\color{blue}compressed message $w_i^0$ we {\em actually} communicate.} This error is then {\em added} to the message $\color{red}\gamma  \nabla f_i(x^1)$ we would normally want to communicate in the {\em next} iteration, providing feedback/compensation for the error incurred. That is, in the next iteration, node $i$ communicates the compressed vector 
$${\color{blue}w_i^{1}} = \cC({\color{green}e^{1}_i} + {\color{red}\gamma \nabla f_i(x^{1}))}$$
instead. Note that since in iteration $1$ we wanted to communicate the vector ${\color{green}e^{1}_i} + {\color{red}\gamma \nabla f_i(x^{1})}$, the error in the next iteration becomes
$${\color{green}e_i^{2}} ={\color{green}e^{1}_i} +  {\color{red}\gamma  \nabla f_i(x^1)} - \cC({\color{green}e^{1}_i} +  {\color{red}\gamma  \nabla f_i(x^1)}) = {\color{green}e^{1}_i} +  {\color{red}\gamma  \nabla f_i(x^1)} - {\color{blue}w_i^1}.$$

This process is repeated, leading to Algorithm~\ref{alg:EF-original}.

			\subsection{Restricted equivalence of \algname{\normalsize EF} and \algname{\normalsize EF21}}
			
We now show that under certain restrictive conditions on the compressor $\cC$, which are not met for compressors used in practice,  \algname{EF} and \algname{EF21} are identical methods.
	\begin{theorem}\label{thm:equivalence}
		Assume that $\cC$ is deterministic, positively homogeneous and additive. Then \algname{EF} (Algorithm~\ref{alg:EF-original}; see appendix) and \algname{EF21} (Algorithm~\ref{alg:EF21})  produce the same sequences of iterates $\{x^t\}_{t\geq 0}$.
	\end{theorem}

	\begin{proof}
		To prove this result, it suffices to show that $w_i^t = \gamma g_i^t $ for all $t \geq 0$.
		We perform this proof by induction. 
		
		{\em Base case ($t=0$):} Recall that $w_i^0 = \cC(\gamma \nabla f_i(x^0))$ and $g_i^0 = \cC(\nabla f_i(x^0))$. By positive homogeneity of $\cC$, we have $$w_i^0 = \cC(\gamma \nabla f_i(x^0)) = \gamma \cC( \nabla f_i(x^0)) = \gamma g_i^0.$$

		{\em Inductive step:} Assume that $w_i^t = \gamma g_i^t $ holds for  some $t \geq 0$. Note that in view of how \algname{EF} operates, we have
\[w_i^{t+1} = \cC\rb{e_i^{t+1} + \gamma \nabla f_i(x^{t+1})}
			= \cC \rb{e_i^t + \gamma \nabla f_i(x^{t}) - w_i^t + \gamma \nabla f_i(x^{t+1}) }.\]
Since we assume that $\cC$					is additive, and because $w_i^t=\cC(e_i^t + \gamma \nabla f_i(x^t))$, we can write
\begin{eqnarray*}w_i^{t+1}
			&=& \cC \rb{e_i^t + \gamma \nabla f_i(x^{t})}  + \cC \rb{ \gamma \nabla f_i(x^{t+1}) - w_i^t } \\
&=& w_i^t  + \cC \rb{ \gamma \nabla f_i(x^{t+1}) - w_i^t } 	.		
			\end{eqnarray*}
Finally, using positive homogeneity, our inductive hypothesis, and the way $g_i^t$ is updated in \algname{EF21}, we can write
\begin{eqnarray*}w_i^{t+1}
		 &=& \gamma \rb{\frac{1}{\gamma} w_i^t + \cC\rb{\nabla f_i(x^{t+1}) - \frac{1}{\gamma} w_i^t }}\notag\\
			&=& \gamma \rb{g_i^t+ \cC\rb{\nabla f_i(x^{t+1}) - g_i^t}}\notag\\
			&=&  \gamma \gitpo,		
			\end{eqnarray*}
 					which concludes our proof.
	\end{proof}

Note that while the Top-$k$ compressor is deterministic and positively homogeneous, it is not additive. Likewise, compressors arising via rescaling of unbiased compressors are randomized, and hence do not satisfy the first condition. Still, the above theorem sheds some (at least to us) unexpected light on the close connection between  \algname{EF} and our new variant, \algname{EF21}. This connection is also what justifies our naming decision: \algname{EF21} -- error feedback mechanism from the year 2021.			

	%%%%%%%%%%%%%%%%%%%%%%%%
	%%%%%%%%%%%%%%%%%%%%%%%%
\section{Experiments} \label{sec:experiments}
%%%%%%%%%%%%%%%%%%%%%%%%
%%%%%%%%%%%%%%%%%%%%%%%%

We first consider solving a logistic regression problem with a non-convex regularizer,  \begin{equation}\label{eq:log-loss-main}\squeeze f(x) = \frac{1}{n} \sum\limits_{i=1}^{n} \log \left(1+\exp \left(-y_{i}a_i^{\top} x\right)\right) + \lam \sum \limits_{j=1}^d \frac{x_j^2}{1 + x_j^2},\end{equation} where $a_{i}  \in \mathbb{R}^{d}, y_{i} \in\{-1,1\}$ are the training data, and $\lambda>0$ is the regularizer parameter. We used $\lambda =0.1$ in all experiments.

\subsection{Datasets, hardware and code}	The datasets were taken from LibSVM \citep{chang2011libsvm}, and were split into $n=20$ equal parts, each associated with one of $20$ clients. The last part, of size $N - 20\cdot\lfloor\nicefrac{N}{20}\rfloor$, was assigned to the last worker. That is, we consider the heterogeneous data distributed regime.  A summary can be found in Table~\ref{tbl:datasets_summary}.   The code was written in Python 3.8 and we used 3 different CPU cluster node types in all experiments (here and in the Appendix): 1) AMD EPYC 7702 64-Core; 2) Intel(R) Xeon(R) Gold 6148 CPU @ 2.40GHz; 3) Intel(R) Xeon(R) Gold 6248 CPU @ 2.50GHz.
	
	\begin{table}[h]
		\caption{Summary of the datasets and splitting of the data among clients.}
		\label{tbl:datasets_summary}
		\centering
		% \footnotesize
		\begin{tabular}{l l r r r}
			\toprule
			Dataset  & $n$ & $N$ (total \# of datapoints) & $d$ (\# of features)  & $N_i$ (\# of datapoints per client)\\
			\midrule
			\texttt{phishing} & 20  &11,055 & 68 & 552 \\ 			
			\texttt{mushrooms} & 20 & 8,120 & 112  & 406\\ 
			\texttt{a9a} & 20  &32,560 & 123  &1,628 \\ 
			\texttt{w8a} & 20  &49,749 & 300 &  2,487\\ 			
			\bottomrule
		\end{tabular}
	\end{table}

\subsection{Experiment 1: Stepsize tolerance}	In our first experiment (see Figure~\ref{fig:3plots}) we  test the robustness/tolerance of  \algname{EF}, \algname{EF21}, and \algname{EF21+} to large stepsizes, using  Top-$k$ with $k=1$ \citep{alistarh2017qsgd} as a canonical example of biased compressor $\cC$. 
 Note that while for all stepsize choices, \algname{EF} gets stuck at a certain accuracy level, \algname{EF21} and \algname{EF21+} do not suffer from this issue, and are hence able to work with  larger or even much larger stepsizes. 
 \begin{figure}[H]
 \centering
	\includegraphics[width=0.9\linewidth]{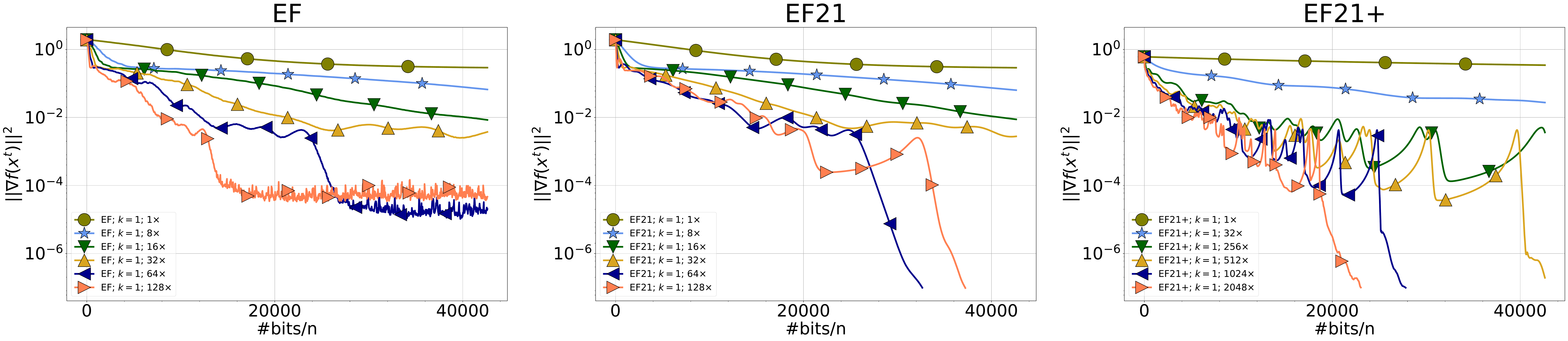} 
	\caption{The performance of \algname{EF}, \algname{EF21}, and \algname{EF21+}  with Top-$1$ compressor, and for increasing stepsizes. Representative dataset used: \texttt{a9a}. By $1\times, 2\times, 4\times$ (and so on) we indicate that the stepsize was set to a multiple of  the largest stepsize predicted by our theory. }\label{fig:3plots}
\end{figure}

\subsection{Experiment 2: Fine-tuning $k$ and the stepsizes}  We now showcase the superior communication efficiency  of \algname{EF21} and \algname{EF21+} over  classical \algname{EF}, again using the Top-$k$ compressor. However, this time we 
fine-tuned $k$ and stepsizes individually  for each methods (details are given in Appendix~\ref{sec:exp_extra}). For reference, we also included distributed gradient descent (\algname{GD}), which can be thought of as \algname{EF21} with $k=d$ (no compression), into the mix.  
%We also include \algname{GD} as a weak baseline for reference, which can be thought of as \algname{EF21} with $k=d$, illustrating that compression indeed helps.

\begin{figure}[H]
	\includegraphics[width=\linewidth]{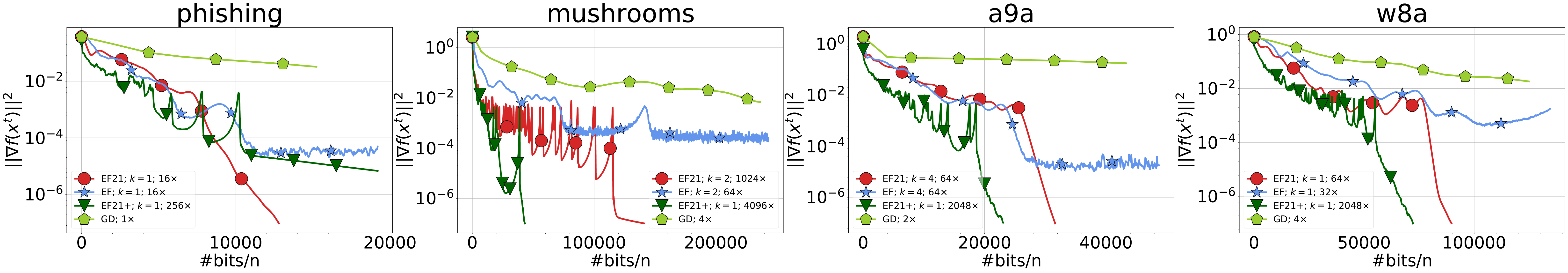} 
%%%%%%%
	\caption{Comparison of \algname{EF21}, \algname{EF21+}  to \algname{EF} with Top-$k$ for individually fine-tuned $k$ and fine-tuned stepsizes for all methods.  }\label{fig:1-in-1}
\end{figure}

In Figure~\ref{fig:1-in-1} we can see that in all cases, the proposed methods outperform \algname{EF} in terms of the \# of bits sent to the server per client on the horizontal axis (${\rm bits}/n$), and rapidly converge to the desired accuracy, whereas \algname{EF} is stuck at some accuracy levels in all cases. Moreover, in all experiments, classical \algname{GD} shows the worst convergence rate.  Note that \algname{EF21} tolerates larger, and \algname{EF21+} much larger, stepsizes than \algname{EF}.

\subsection{Further experiments} Further experiments, including deep learning experiments, are presented in Appendix~\ref{sec:exp_extra}.

	\bibliographystyle{plainnat}	
	\bibliography{bibliography.bib}

	\appendix
	
	\part*{Appendix}

	%%%%%%%%%%%%%%%%%%%%%%%%
	%%%%%%%%%%%%%%%%%%%%%%%%
	% \clearpage
	\section{Extra Experiments}\label{sec:exp_extra}
%%%%%%%%%%%%%%%%%%%%%%%%
%%%%%%%%%%%%%%%%%%%%%%%%

We now present several additional experiments. First, in Section~\ref{sec:exp:logistic} we comment on experiments with nonconvex logistic regression (see \eqref{eq:log-loss-main}), in Section~\ref{sec:exp-LS}  we perform experiments on least-squares (as an example of a function that is not strongly convex but satisfies the PL inequality),  and finally, in Section~\ref{sec:exp-DL} we conduct several deep learning experiments.

\subsection{Experiments with nonconvex logistic regression} \label{sec:exp:logistic}

\subsubsection{Experiment 1: Stepsize tolerance (extension)}\label{sec:stepsize_tollerance_extension}

This sequence of experiments extends the results presented in the corresponding paragraph of Section~\ref{sec:experiments}. 
For each dataset, we select the parameter $k$ (varied by rows) within the powers of $2$. For each plot, we vary the stepsize within the powers of $2$ starting from the largest theoretically accepted $\g$.

For example, for $k=2$ and \algname{EF21+} with \texttt{mushrooms} dataset we consider factors from the set $$\cb{1, 2, 4, 8, 16, 32, 64, 128, 256, 512, 1024, 2048}$$ and select the stepsize as a multiple of the upper bound stated in Theorem~\ref{thm:main-distrib}. 

Red diamond markers indicate the iterations at which \algname{EF21+} method uses mostly \algname{DCGD} steps. Precisely, the red diamond marker appears on the plot if the distortion  $\norm{\cC(s) - s}$ is smaller that $\norm{\cM(s) - s}$ for at least half of the workers, where $s=\nabla f_i(x^{t+1})$.
For more details, see figures below, where parameter $k$ is fixed within each row and each column corresponds to a particular method.

All of the figures above illustrate that \algname{EF21} and \algname{EF21+} tolerates much larger stepsizes, which makes them more efficient in practice. Moreover, in all experiments with large stepsizes ($16\times$--$128\times$), \algname{EF} starts oscillating, which hinders the convergence to the desired tolerance

\begin{figure}[H]
\centering
	\includegraphics[width=0.8\linewidth]{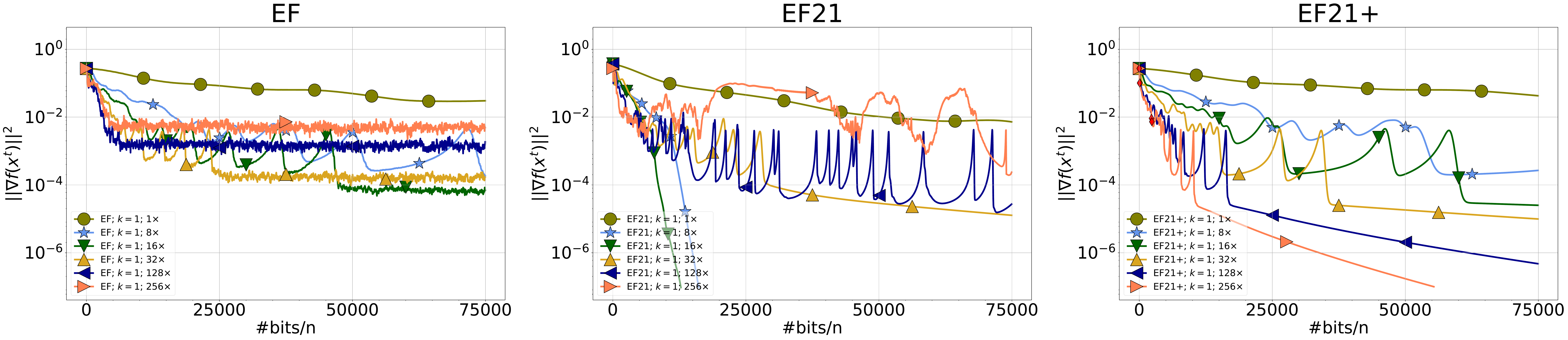} 
	\includegraphics[width=0.8\linewidth]{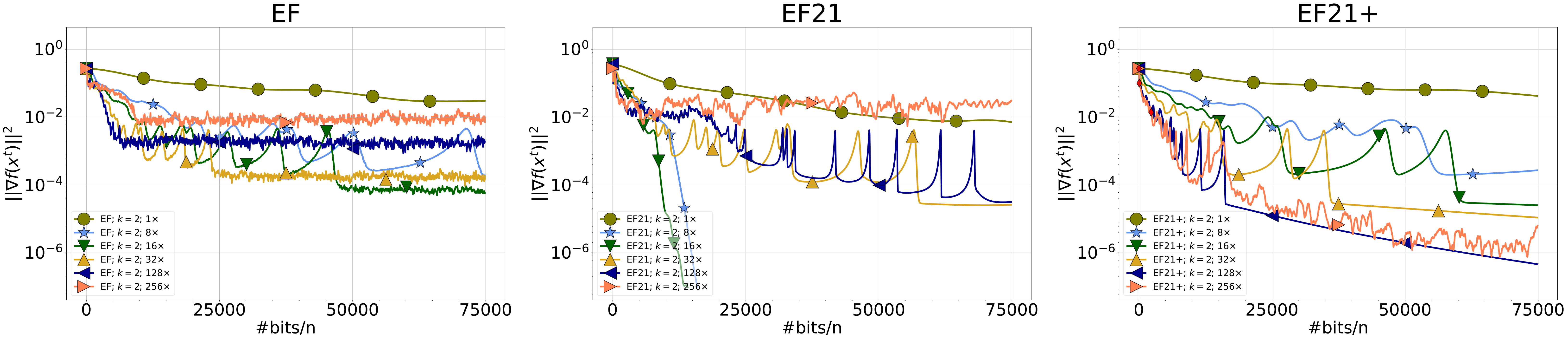}
	\includegraphics[width=0.8\linewidth]{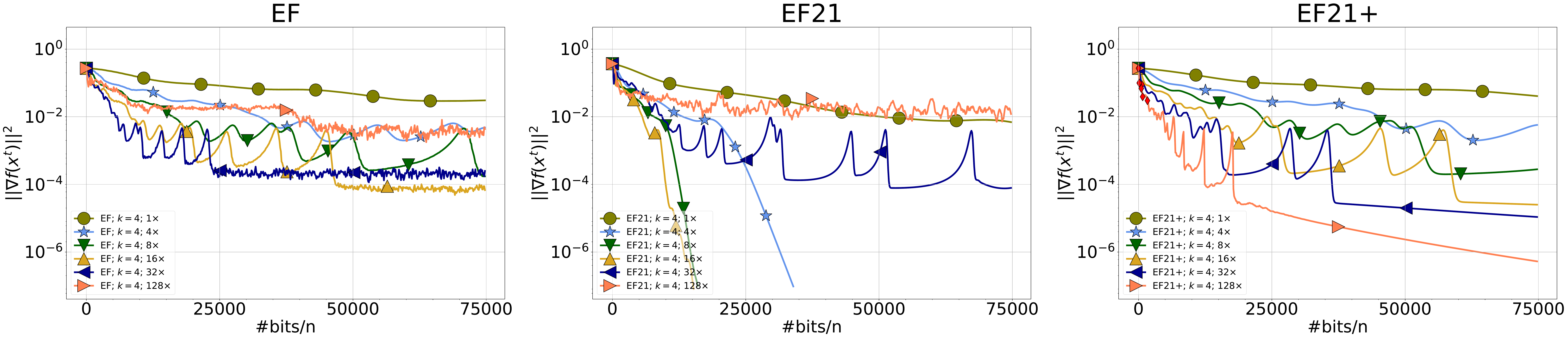}
	\includegraphics[width=0.8\linewidth]{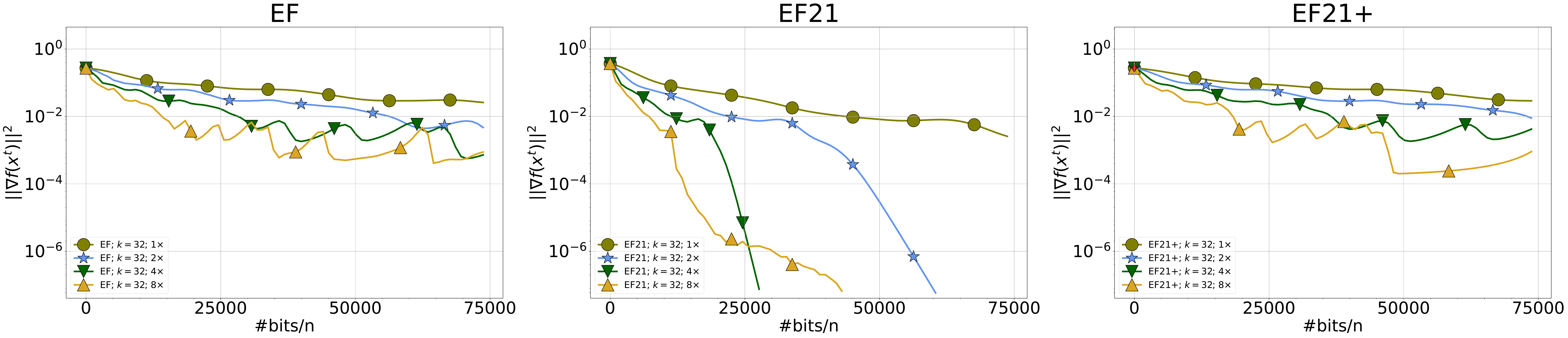}   
	%%%%%%%
	%\hline
	%%%%%%%
	
	\caption{The performance of \algname{EF}, \algname{EF21}, and \algname{EF21+}  with Top-$k$ compressor, and for increasing stepsizes. Each row corresponds to a different value of $k \in \cb{1, 2, 4, 32}$. The dataset used: \texttt{phishing}. By $1\times, 2\times, 4\times$ (and so on) we indicate that the stepsize was set to a multiple of  the largest stepsize predicted by our theory.} \label{fig:3plots_more_phishing}
\end{figure}

\begin{figure}[H]
\centering
	\includegraphics[width=0.8\linewidth]{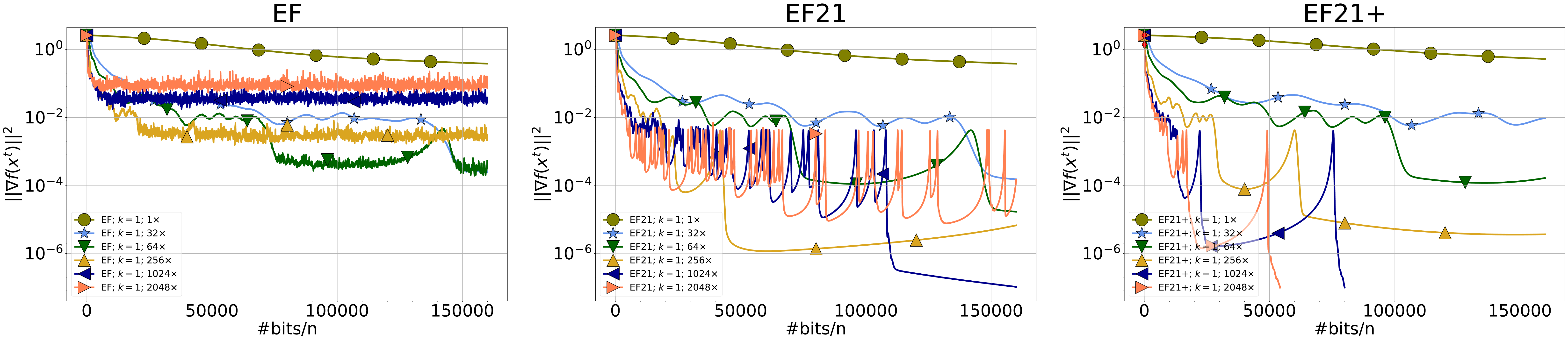}
	\includegraphics[width=0.8\linewidth]{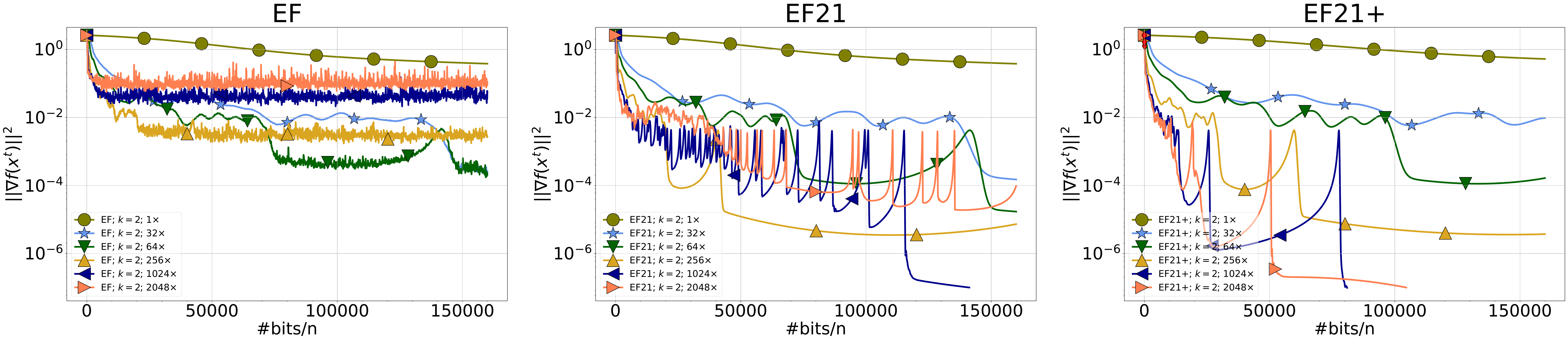}
	\includegraphics[width=0.8\linewidth]{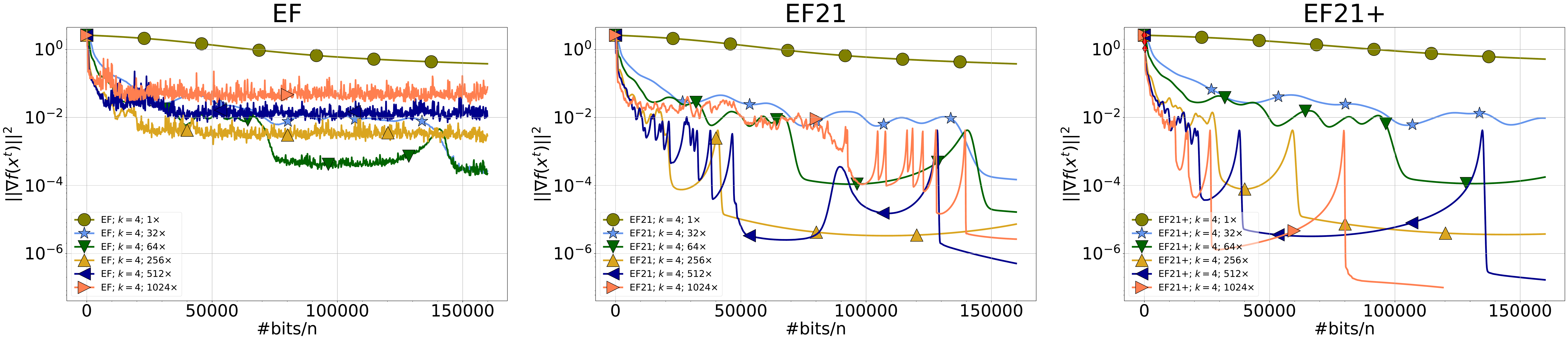}
	\includegraphics[width=0.8\linewidth]{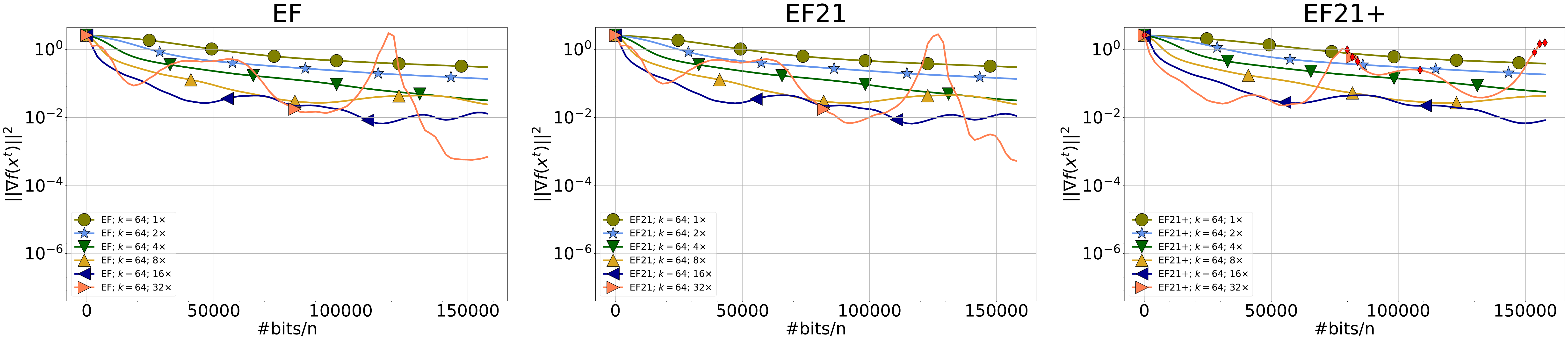}
	%%%%%%%
	%\hline
	%%%%%%%
	
	\caption{The performance of \algname{EF}, \algname{EF21}, and \algname{EF21+}  with Top-$k$ compressor, and for increasing stepsizes. Each row corresponds to a different value of $k \in \cb{1, 2, 4, 64}$. The dataset used: \texttt{mushrooms}. By $1\times, 2\times, 4\times$ (and so on) we indicate that the stepsize was set to a multiple of  the largest stepsize predicted by our theory.} \label{fig:3plots_more_mushrooms}
\end{figure}

\begin{figure}[H]
\centering
	\includegraphics[width=0.8\linewidth]{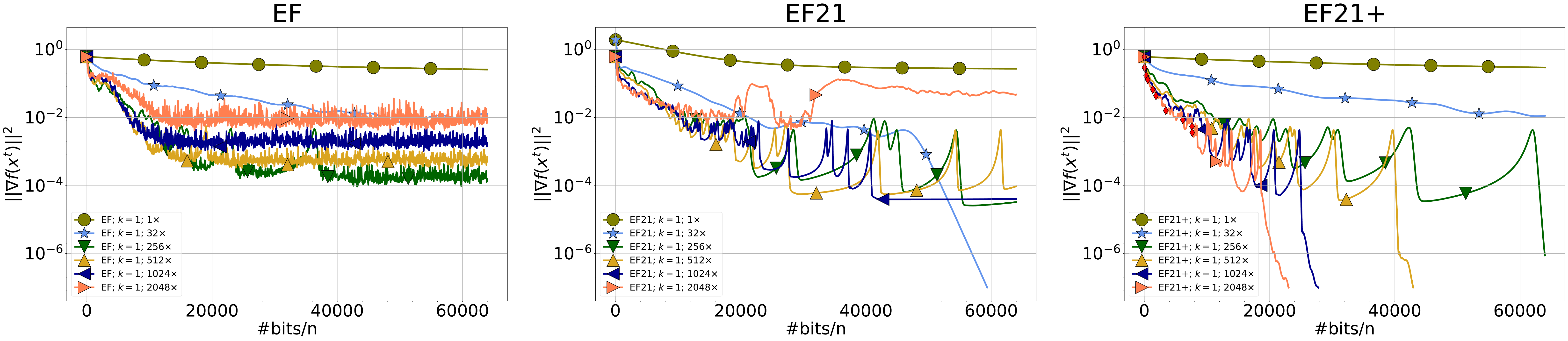}
	\includegraphics[width=0.8\linewidth]{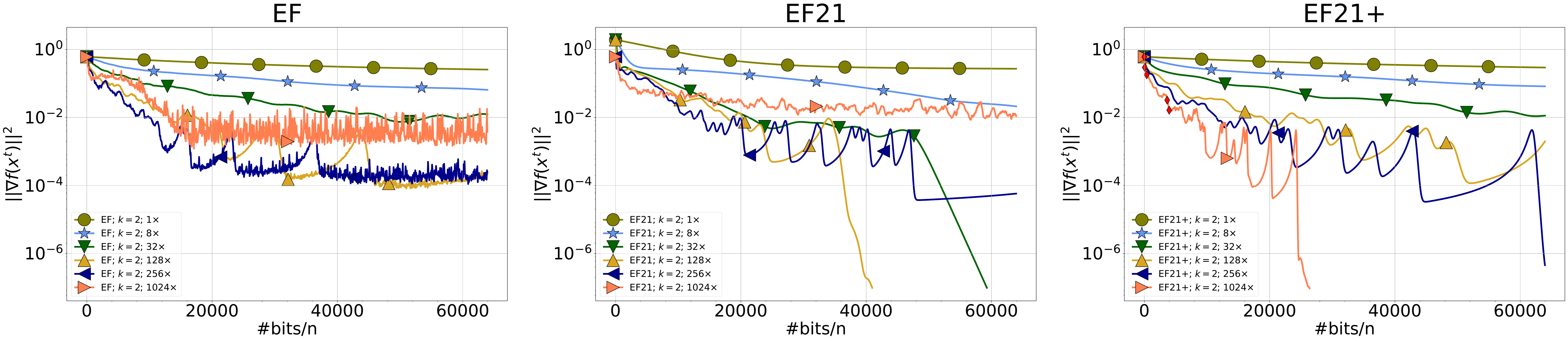}
	\includegraphics[width=0.8\linewidth]{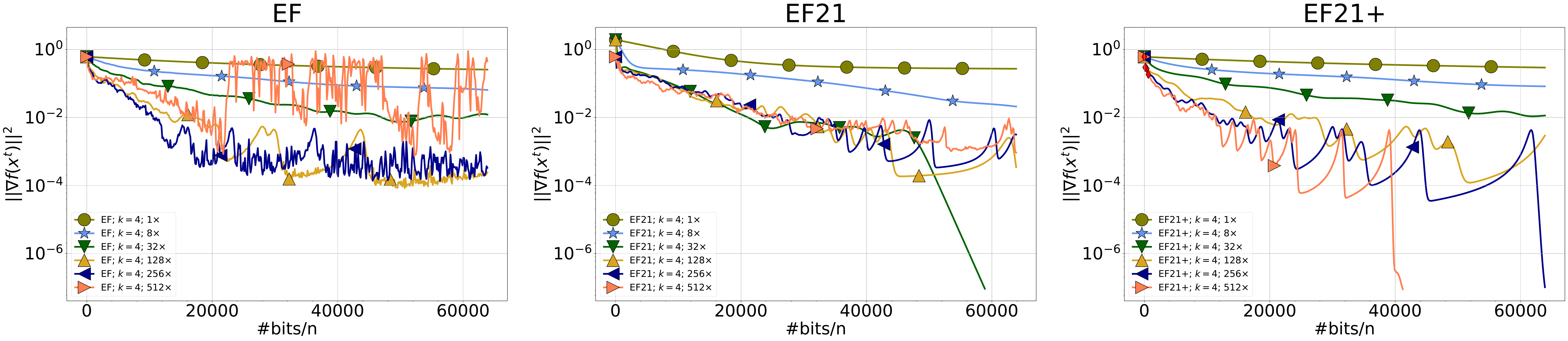}
	\includegraphics[width=0.8\linewidth]{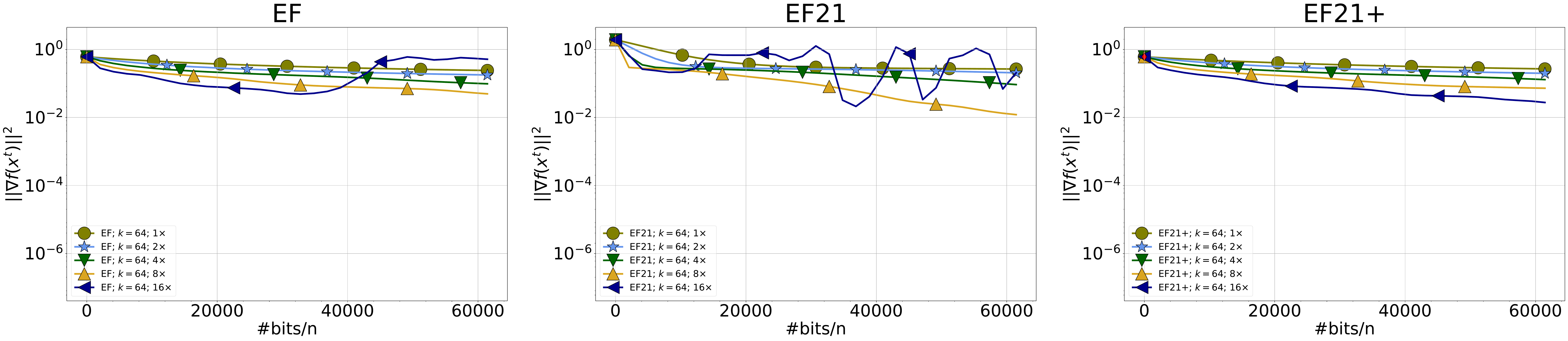}
	%%%%%%%
	%\hline
	%%%%%%%
	
	\caption{The performance of \algname{EF}, \algname{EF21}, and \algname{EF21+}  with Top-$k$ compressor, and for increasing stepsizes. Each row corresponds to a different value of $k \in \cb{1, 2, 4, 64}$. The dataset used: \texttt{a9a}. By $1\times, 2\times, 4\times$ (and so on) we indicate that the stepsize was set to a multiple of  the largest stepsize predicted by our theory.} \label{fig:3plots_more_a9a}
\end{figure}

\begin{figure}[H]
\centering
	\includegraphics[width=0.8\linewidth]{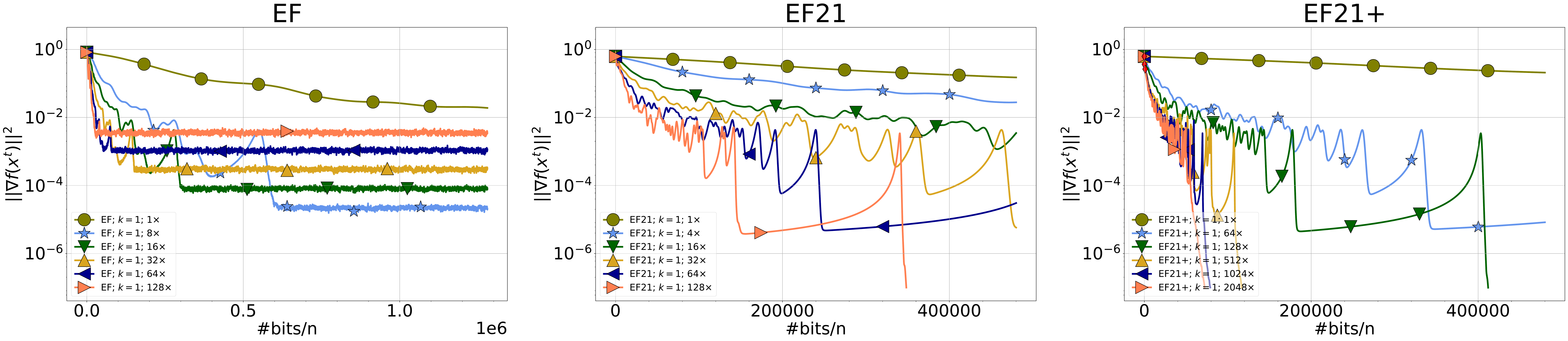}
	\includegraphics[width=0.8\linewidth]{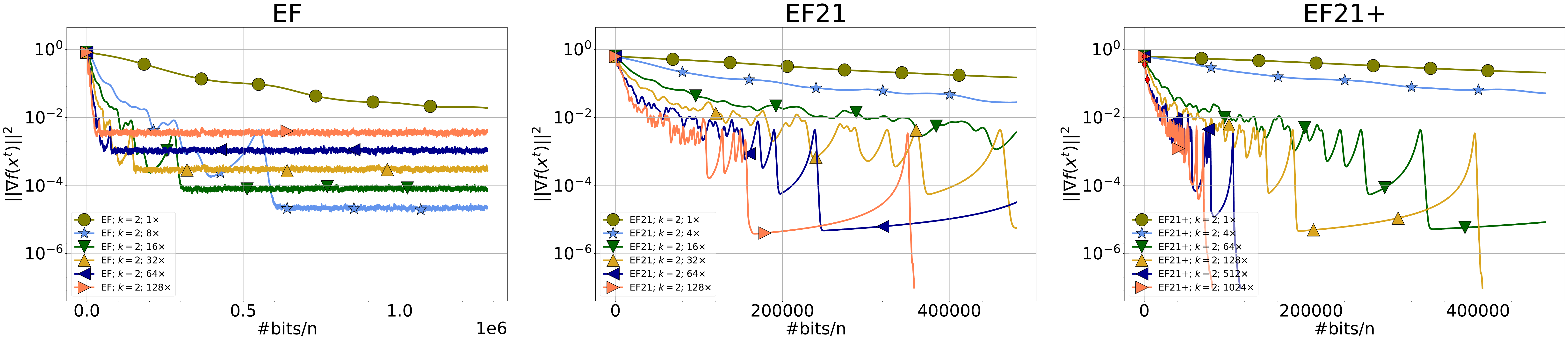}
	\includegraphics[width=0.8\linewidth]{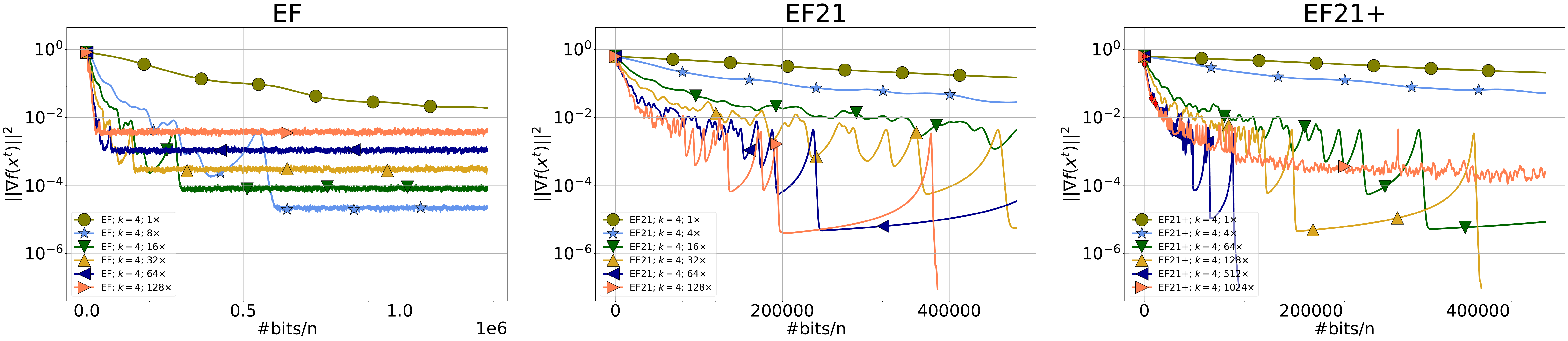}
	\includegraphics[width=0.8\linewidth]{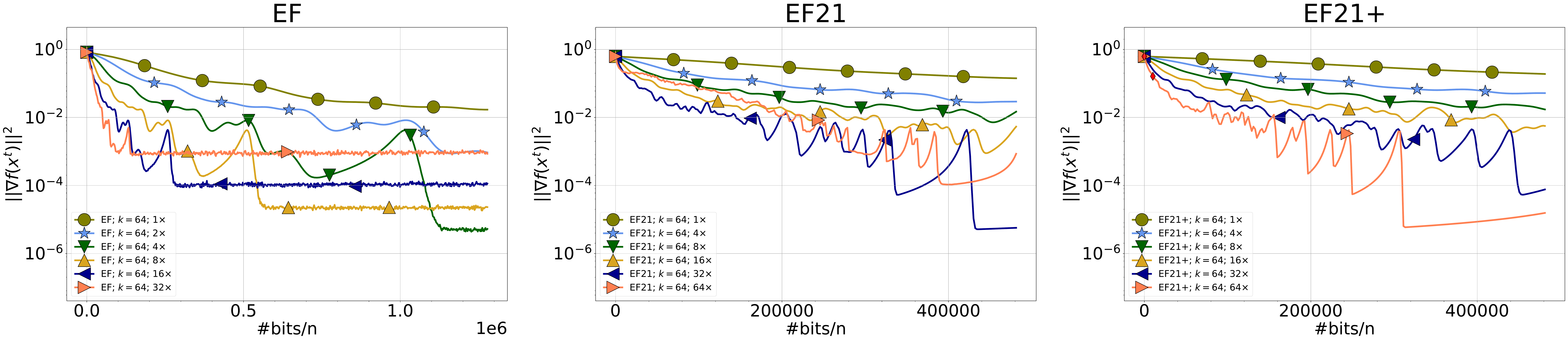}
	%%%%%%%
	%\hline
	%%%%%%%
	
	\caption{The performance of \algname{EF}, \algname{EF21}, and \algname{EF21+}  with Top-$k$ compressor, and for increasing stepsizes. Each row corresponds to a different value of $k \in \cb{1, 2, 4, 64}$. The dataset used: \texttt{w8a}. By $1\times, 2\times, 4\times$ (and so on) we indicate that the stepsize was set to a multiple of  the largest stepsize predicted by our theory.} \label{fig:3plots_more_w8a}
\end{figure}
%%%%%%%

%%%%%%%

\subsubsection{Experiment 2: Fine-tuning $k$ and the stepsizes (extension)}
%\begin{figure}[H]
%\includegraphics[width=\linewidth]{plots/_0_biased-diana-full-grad-ef-n_20-k_1-2-4-1-loss_log-reg_2} 
%\includegraphics[width=\linewidth]{plots/_0_biased-diana-full-grad-ef-n_20-k_1-2-4-1-loss_log-reg} 
%\caption{{\bf Rows 1-2:} Stepsize fine-tuning for \algname{EF21} and \algname{EF}.  For each method and dataset, we choose the best $k$, and vary stepsize for this $k$. {\bf Rows 3-4:}}
%\end{figure}
This sequence of experiments extends the results presented in the similar paragraph of Section~\ref{sec:experiments}. In these plots we focus on the effect of the parameter $k$ on convergence. For each method, dataset, and $k$, the stepsize is fine-tuned
(based on the fine-tuning results from Section~\ref{sec:stepsize_tollerance_extension}). Note that the theoretical stepsize allowed by Theorem~\ref{thm:main-distrib}  increases by itself with the increase of  $k$.

\begin{figure}[H]
	\includegraphics[width=\linewidth]{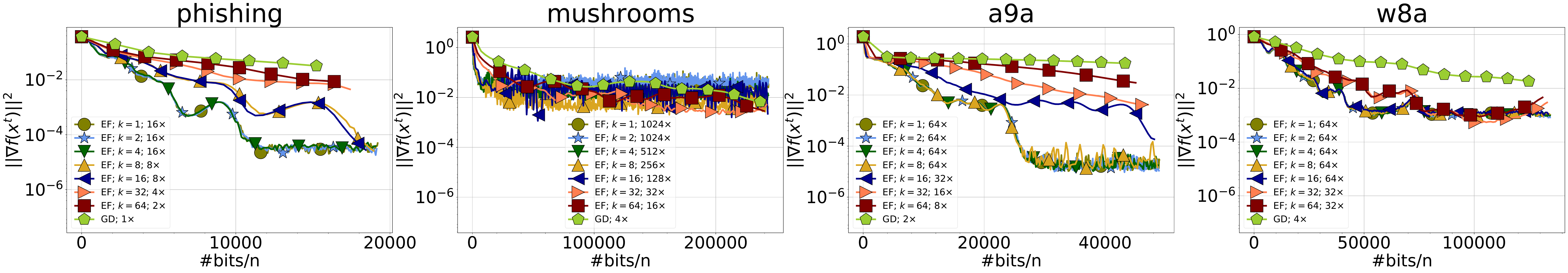} 
	\includegraphics[width=\linewidth]{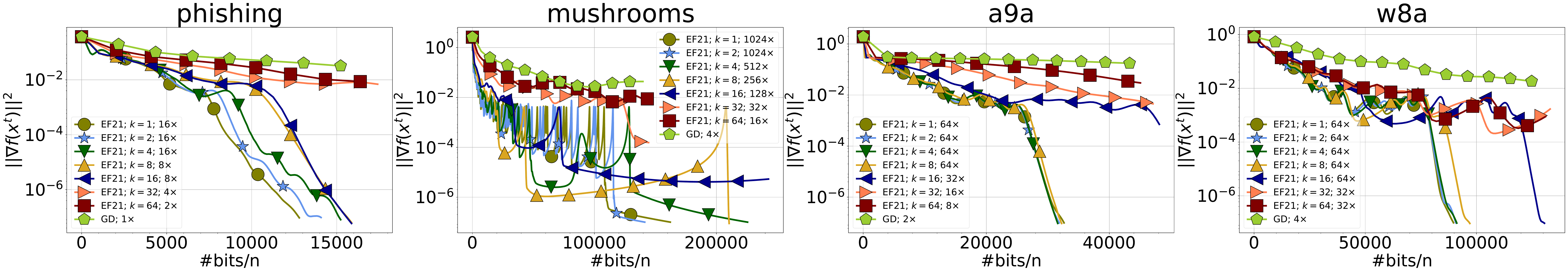} 
	%%%%%%%
	\caption{Effect of the parameter $k$ on convergence. For each method, dataset and $k$ the stepsize is fine-tuned. By $1\times, 2\times, 4\times$ (and so on) we indicate that the stepsize was set to a multiple of  the largest stepsize predicted by our theory.} \label{fig:4-in-1}
\end{figure}

%In Figure~\ref{fig:5-in-1} (Rows 1-2), for each $k$, the fine-tuned stepsizes are applied (e.g., $8\times$ means that we choose the stepsize which is $8$ times larger than the maximal value prescribed by Theorem~\ref{thm:main-distrib} for certain $k$). 
\begin{figure}[H]
	\includegraphics[width=\linewidth]{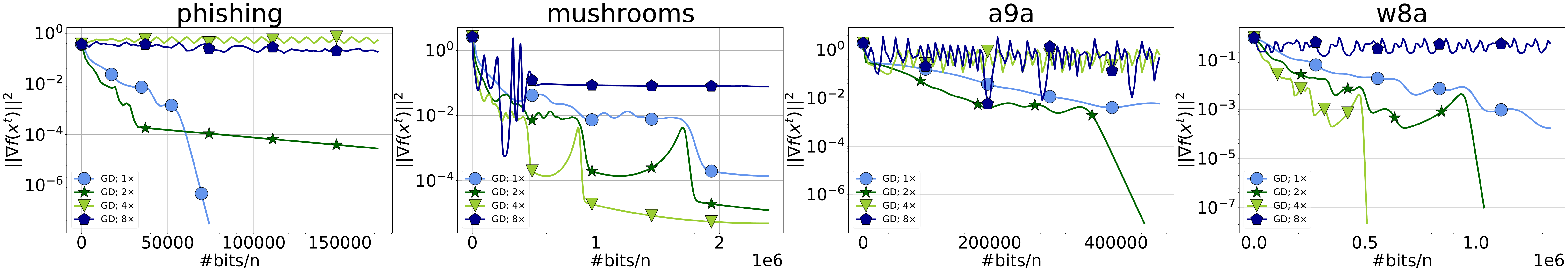} 
	%%%%%%%
	\caption{GD tuning.} \label{fig:GD_fine_tuning}
\end{figure}

We see that the best choice of $k$ relates to $1, 2$ or $4$, which confirms that both \algname{EF21} and \algname{EF} are more communication efficient compared to \algname{GD}.

%{\bf Robustness to even larger stepsizes.}
%Next, we are interested in how robust \algname{EF21} and \algname{EF} are to the stepsizes beyond the predicted theory. Thus, we fix the fine-tuned $k$ (from the previous experiment) for each method and dataset. Figure~\ref{fig:step_size_tuning} illustrates that \algname{EF21} tolerates much larger stepsizes, which makes it more efficient in practice. Moreover, in all experiments with large stepsizes ($16\times$--$128\times$), \algname{EF} starts oscillating, which hinders the convergence to the desired tolerance.

\subsection{Experiments with least squares} \label{sec:exp-LS}

In this section we will test on a function satisfying the PL condition (see Assumption \ref{ass:PL}). In particular, we consider the function
$$ f(x) = \frac{1}{n}\sum\limits_{i=1}^{n}(a_{i}^{\top}x  - b_i )^2 ,$$ where $a_{i}  \in \mathbb{R}^{d}, y_{i} \in\{-1,1\}$ are the training data.	We consider the same datasets as for the logistic regression problem.

\subsubsection{ Experiment 1: Stepsize tolerance}
In this set of experiments we test the robustness/tolerance of \algname{EF}, \algname{EF21}, and \algname{EF21+} to large stepsizes, using  Top-$k$ \citep{alistarh2017qsgd} as a canonical example of biased compressor $\cC$. 
For each plot, we vary the stepsize within the powers of $2$ starting from the largest theoretically accepted $\g$.
For example, for $k=2$ and \algname{EF21+} with \texttt{mushrooms} dataset we consider factors from the set $\cb{1, 4, 64, 256, 1024}$ and select the stepsize as a multiple of the upper bound stated in Theorem~\ref{thm:main-distrib}. Red diamond markers indicate the iterations at which \algname{EF21+} method uses mostly \algname{DCGD} steps. More precisely, the red diamond marker appears on the plot if the distortion  $\norm{\cC(s) - s}$ is smaller that $\norm{\cM(s) - s}$ for at least half of the workers, where $s=\nabla f_i(x^{t+1})$.
For more details, see Figures~\ref{fig:3plots_more_phishing_PL}--\ref{fig:3plots_more_w8a_PL}, where parameter $k$ is fixed within each row and each column correspond to a particular method.
\begin{figure}[H]
\centering
	\includegraphics[width=0.8\linewidth]{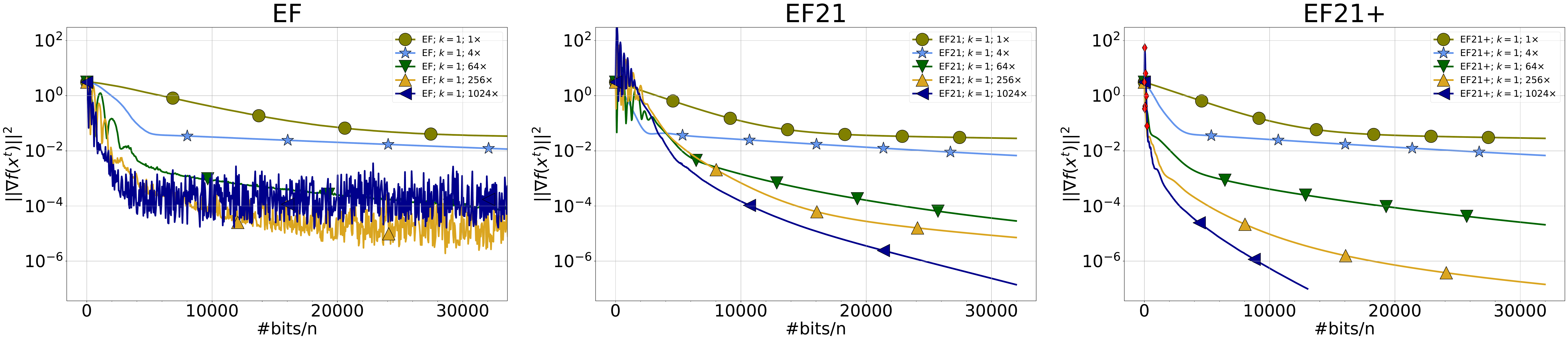}
	\includegraphics[width=0.8\linewidth]{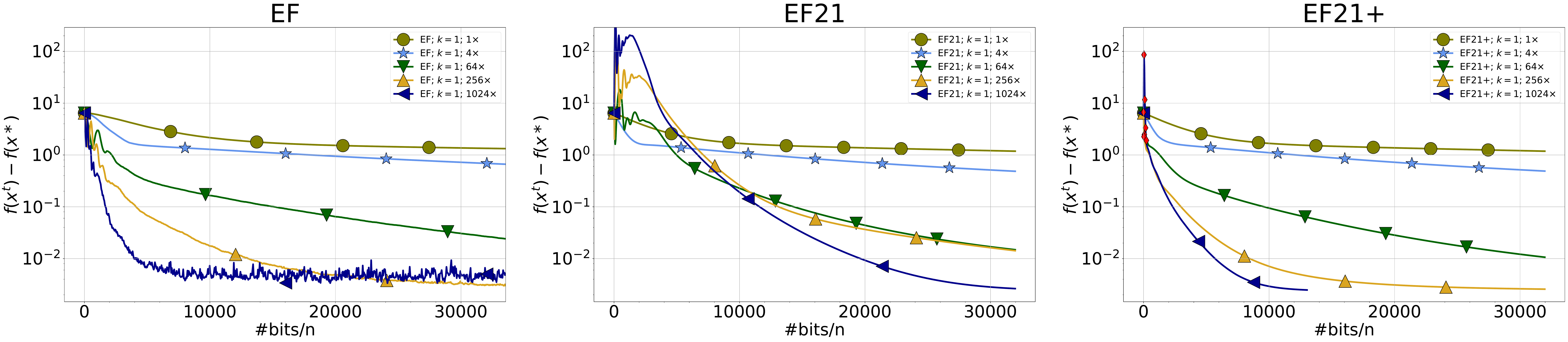}
	\includegraphics[width=0.8\linewidth]{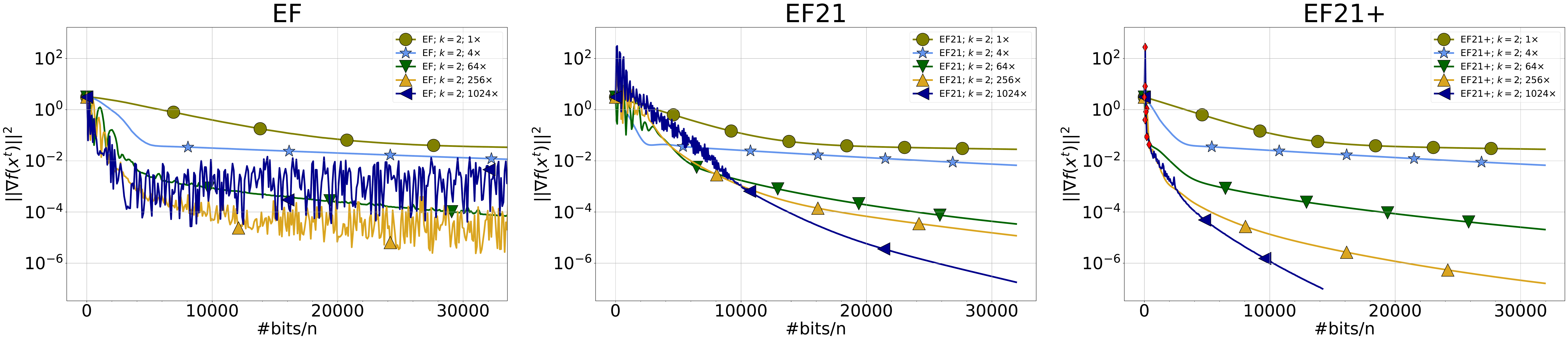}
	\includegraphics[width=0.8\linewidth]{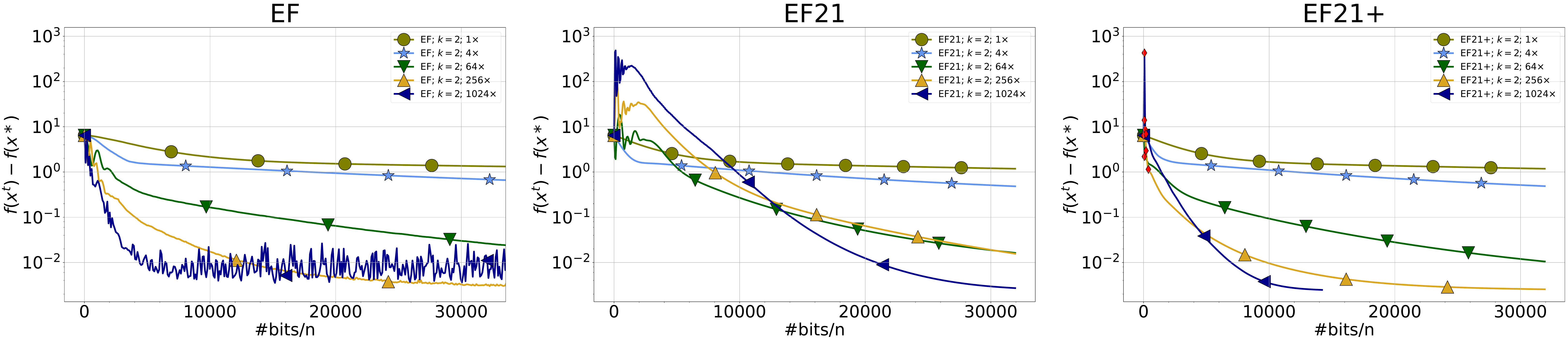}
	\includegraphics[width=0.8\linewidth]{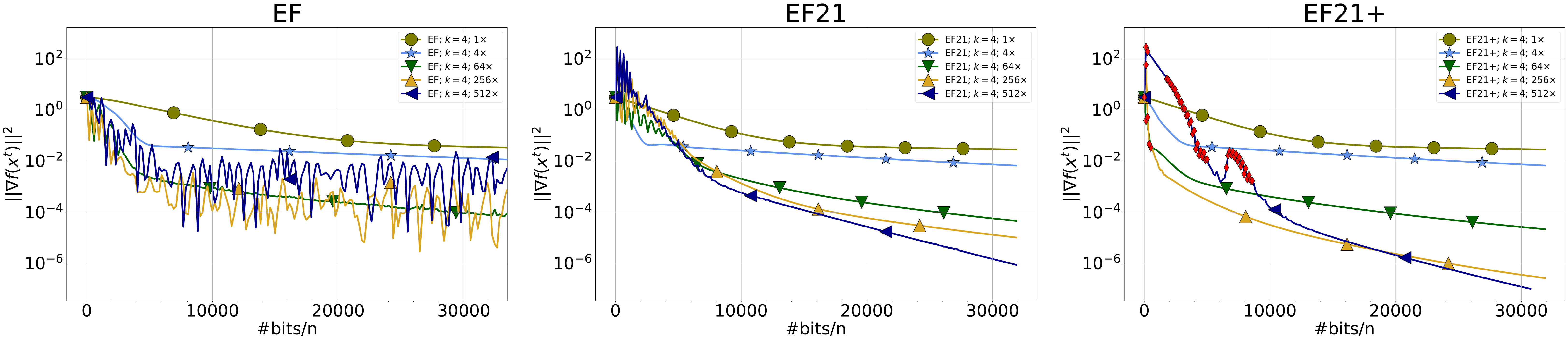}
	\includegraphics[width=0.8\linewidth]{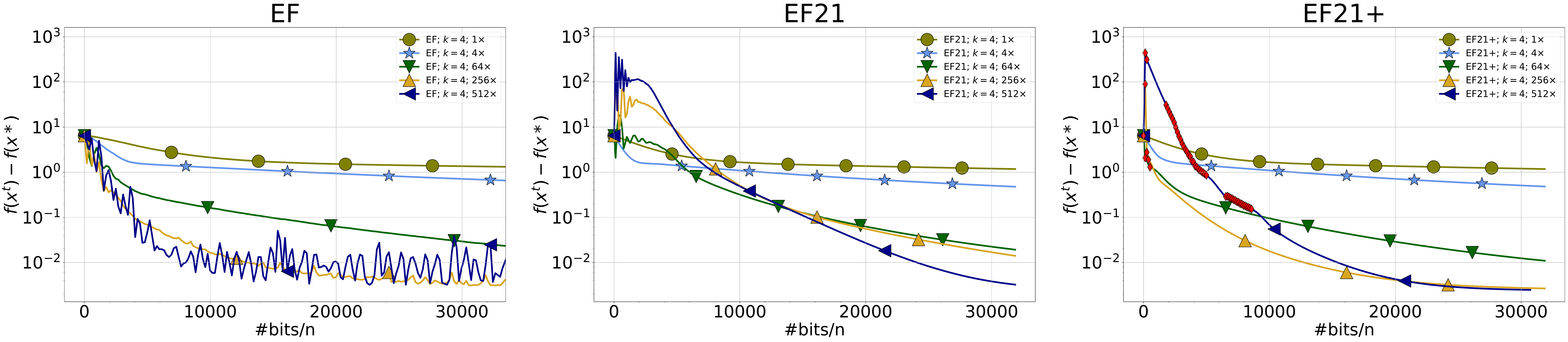}
	%%%%%%%
	%\hline
	%%%%%%%
	
	\caption{The performance of \algname{EF}, \algname{EF21}, and \algname{EF21+}  with Top-$k$ compressor, and for increasing stepsizes. The dataset used: \texttt{phishing}. By $1\times, 2\times, 4\times$ (and so on) we indicate that the stepsize was set to a multiple of  the largest stepsize predicted by our theory.} \label{fig:3plots_more_phishing_PL}
\end{figure}

\begin{figure}[H]
\centering
	\includegraphics[width=0.8\linewidth]{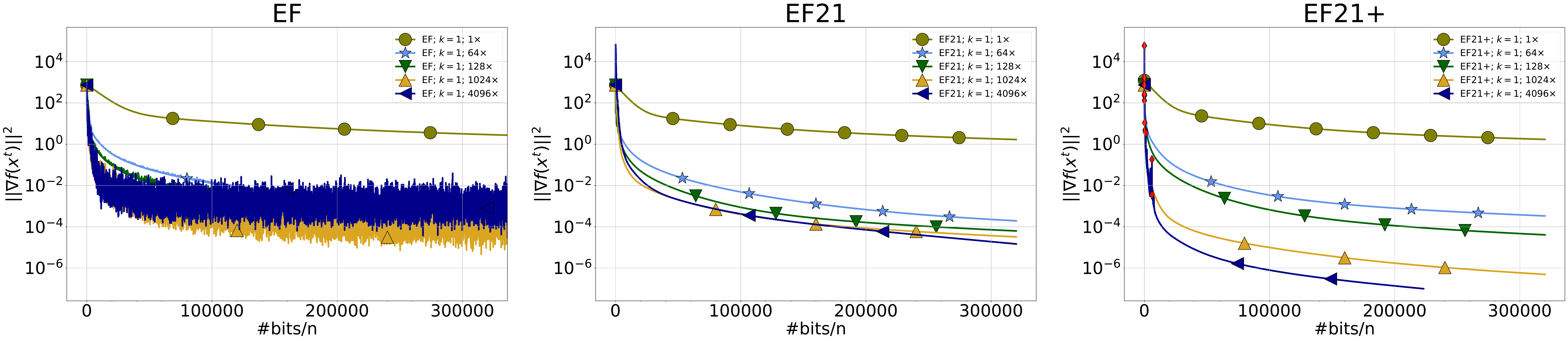}
	\includegraphics[width=0.8\linewidth]{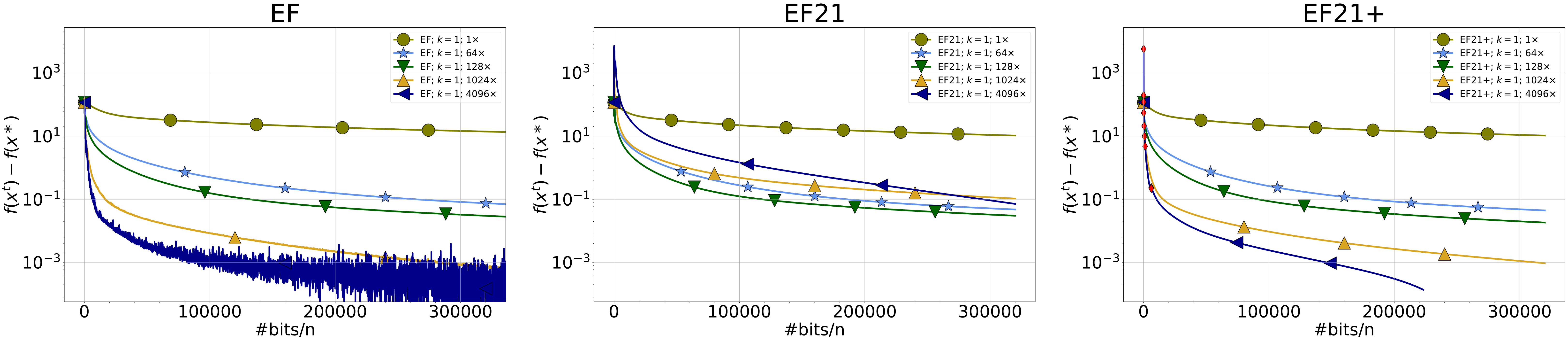}
	\includegraphics[width=0.8\linewidth]{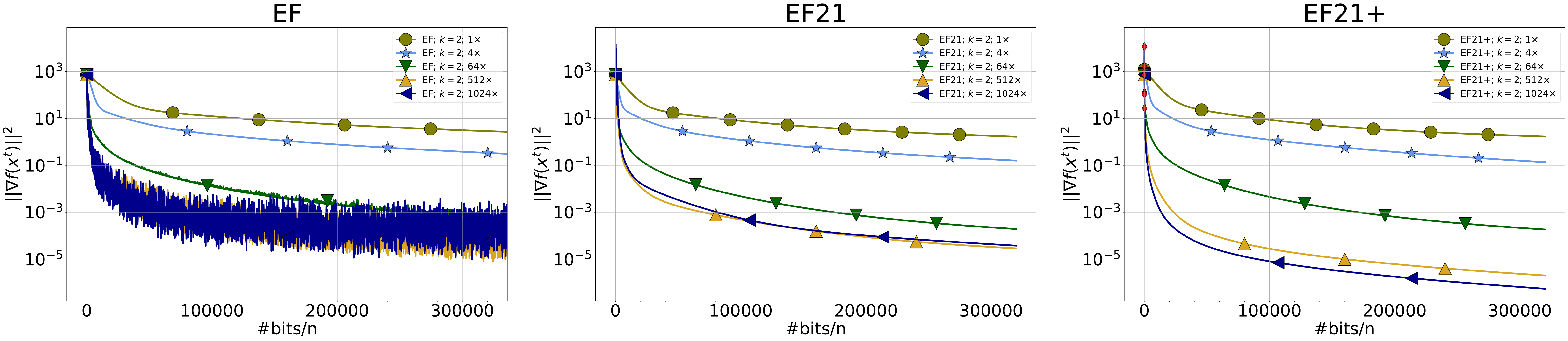}
	\includegraphics[width=0.8\linewidth]{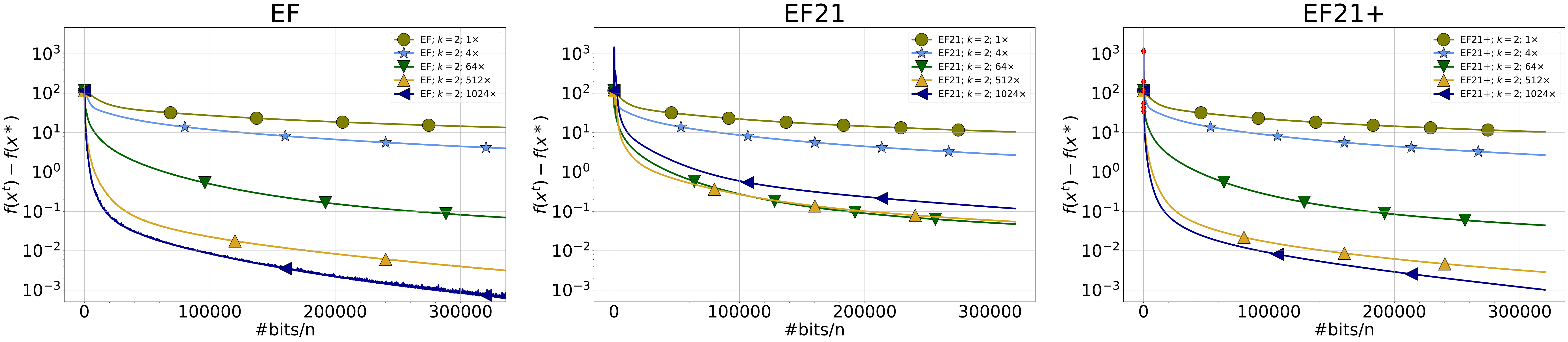}
	\includegraphics[width=0.8\linewidth]{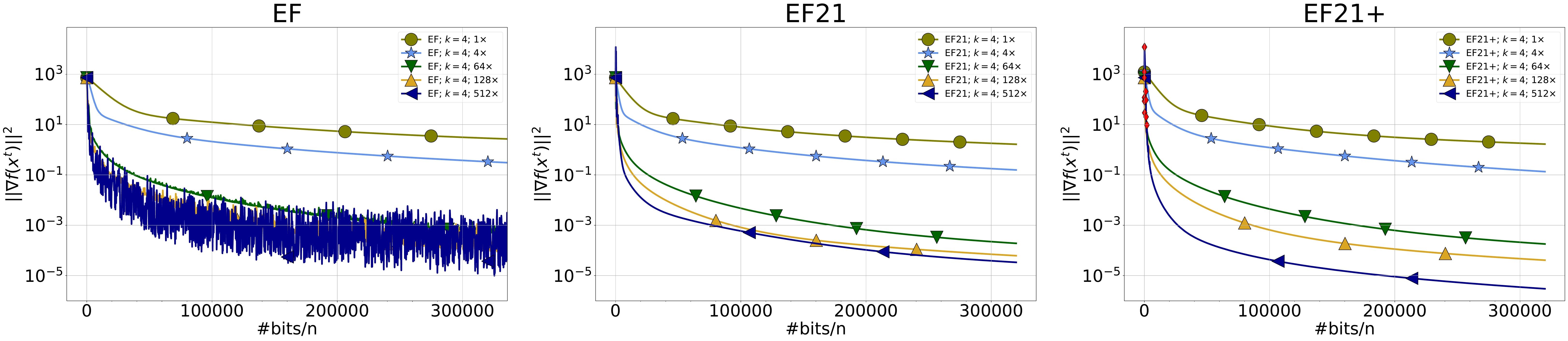}
	\includegraphics[width=0.8\linewidth]{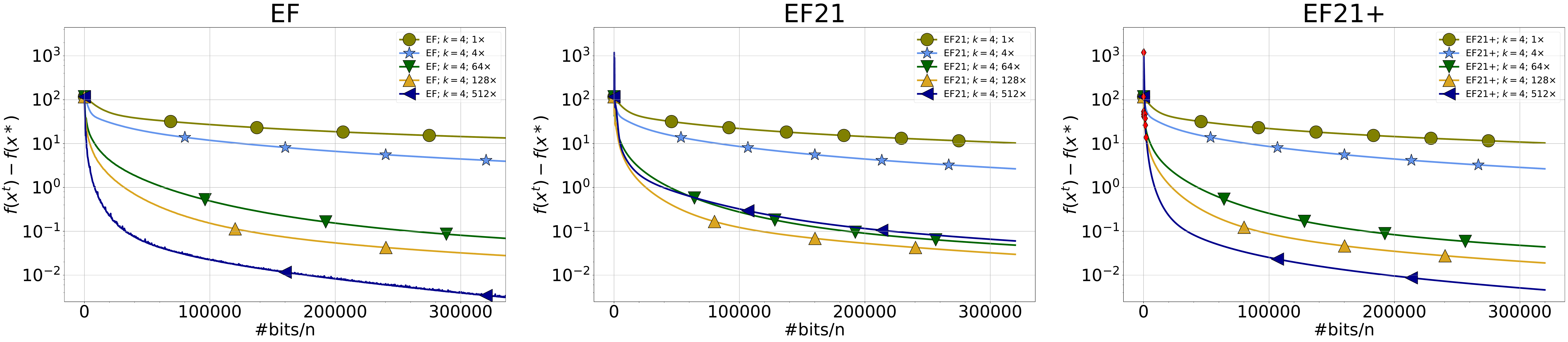}
	%%%%%%%
	%\hline
	%%%%%%%
	
	\caption{The performance of \algname{EF}, \algname{EF21}, and \algname{EF21+}  with Top-$k$ compressor, and for increasing stepsizes. The dataset used: \texttt{mushrooms}. By $1\times, 2\times, 4\times$ (and so on) we indicate that the stepsize was set to a multiple of  the largest stepsize predicted by our theory.} \label{fig:3plots_more_mushrooms_PL}
\end{figure}

\begin{figure}[H]
\centering
	\includegraphics[width=0.8\linewidth]{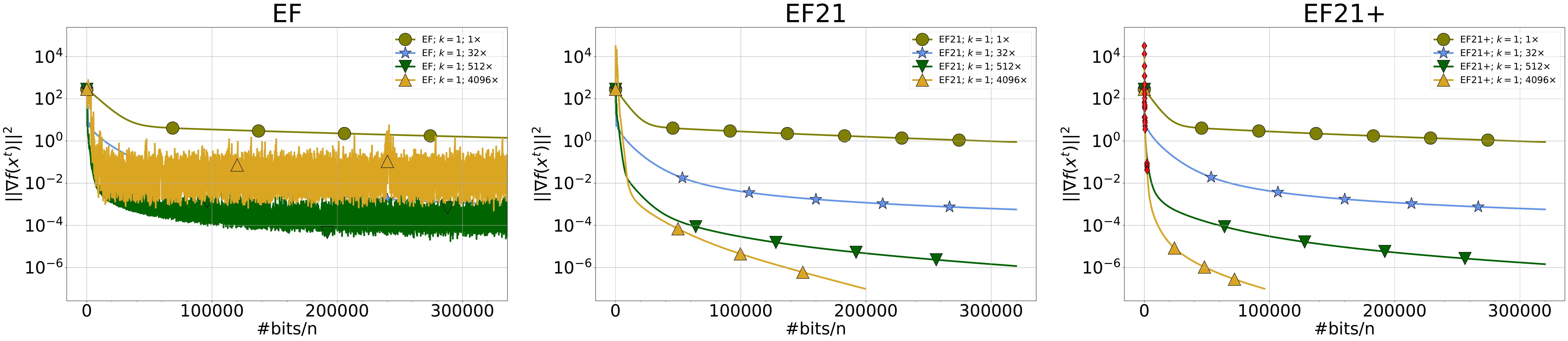}
	\includegraphics[width=0.8\linewidth]{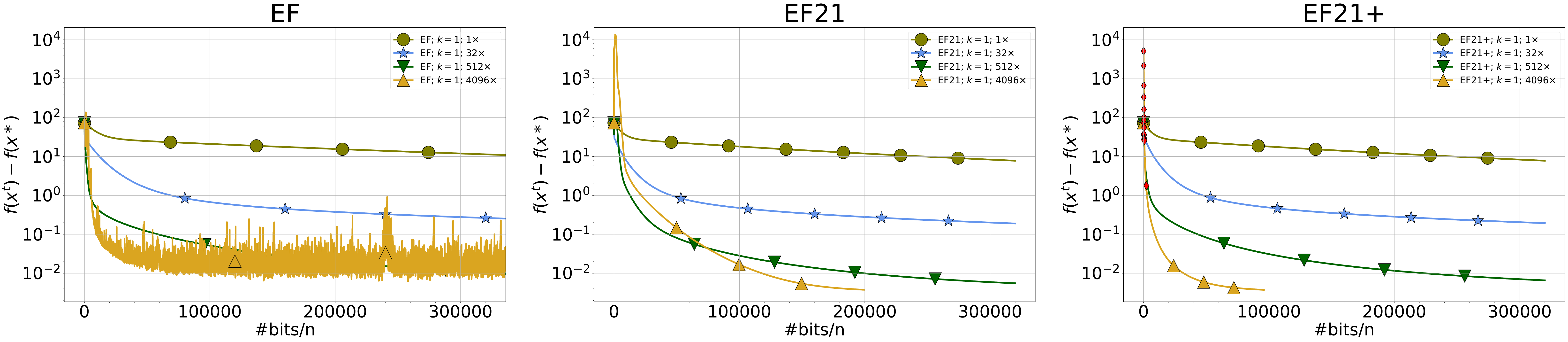}
	\includegraphics[width=0.8\linewidth]{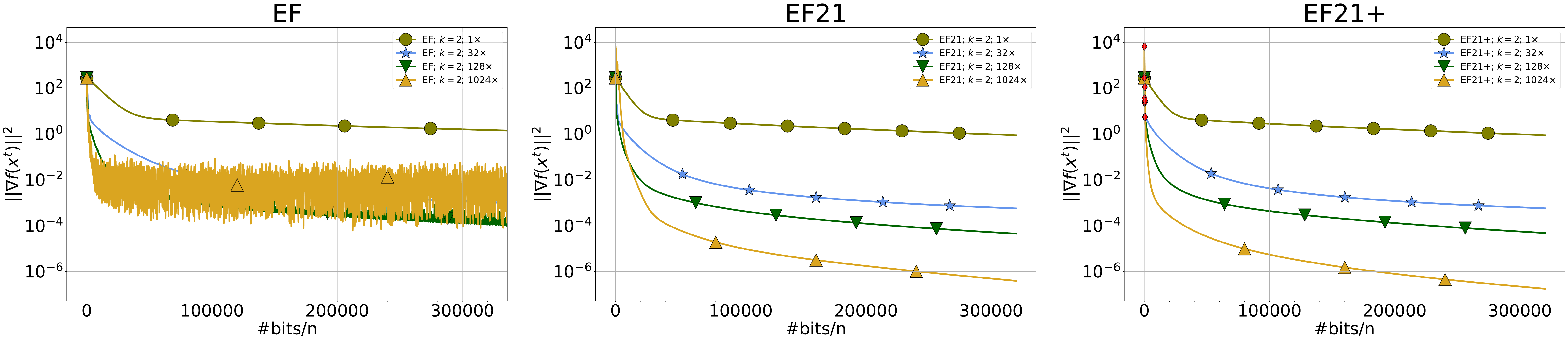}
	\includegraphics[width=0.8\linewidth]{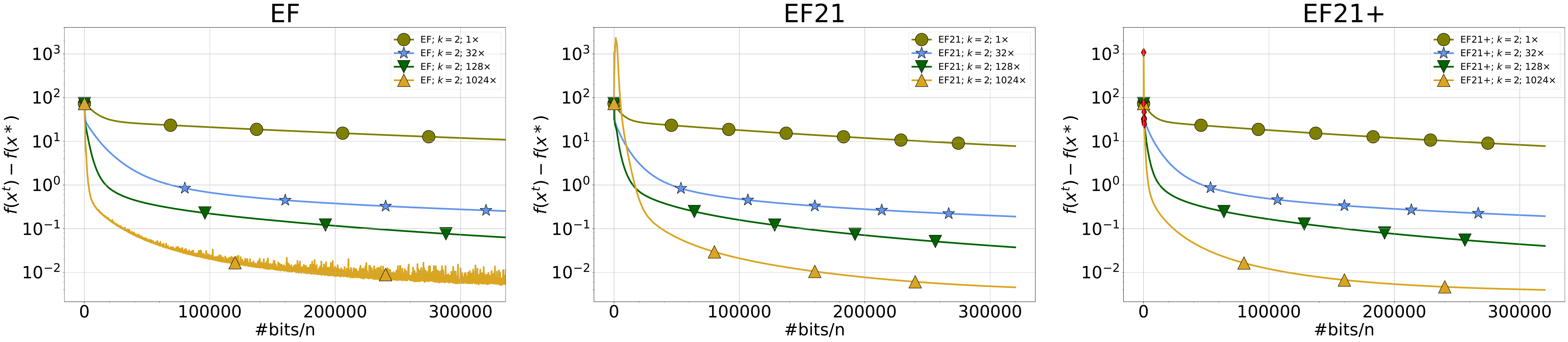}
	\includegraphics[width=0.8\linewidth]{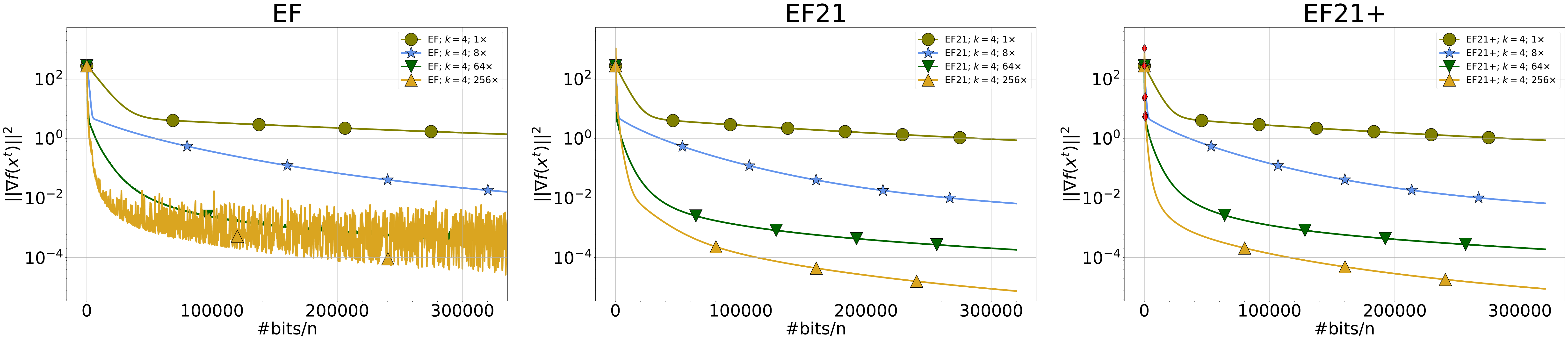}
	\includegraphics[width=0.8\linewidth]{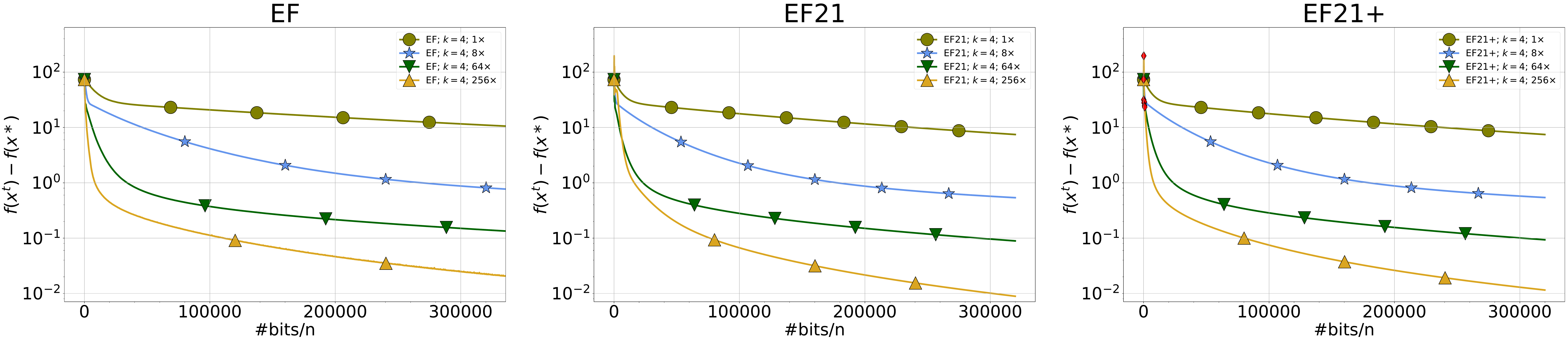}
	%%%%%%%
	%\hline
	%%%%%%%
	
	\caption{The performance of \algname{EF}, \algname{EF21}, and \algname{EF21+}  with Top-$k$ compressor, and for increasing stepsizes. The dataset used: \texttt{a9a}. By $1\times, 2\times, 4\times$ (and so on) we indicate that the stepsize was set to a multiple of  the largest stepsize predicted by our theory.} \label{fig:3plots_more_a9a_PL}
\end{figure}

\begin{figure}[H]
\centering
	\includegraphics[width=0.8\linewidth]{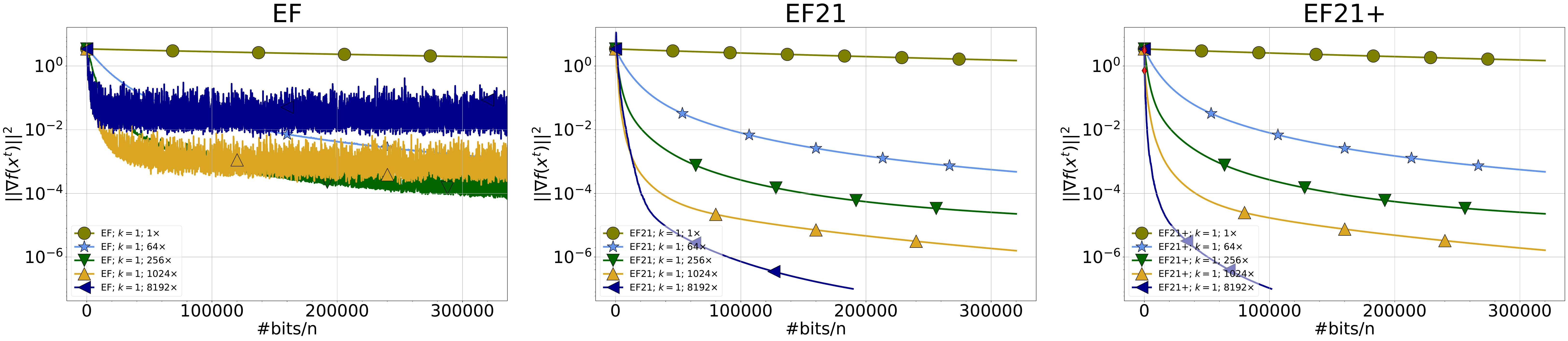}
	\includegraphics[width=0.8\linewidth]{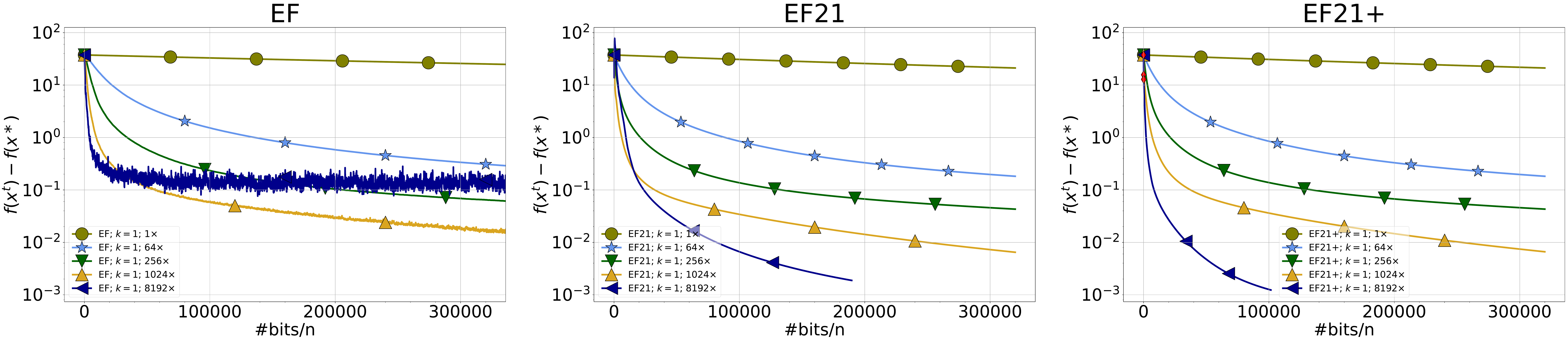}
	\includegraphics[width=0.8\linewidth]{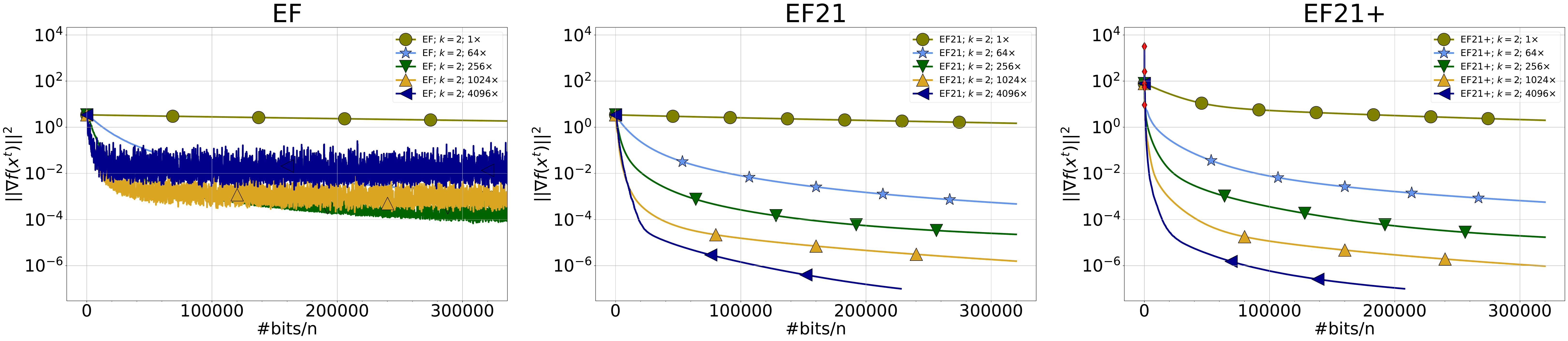}
	\includegraphics[width=0.8\linewidth]{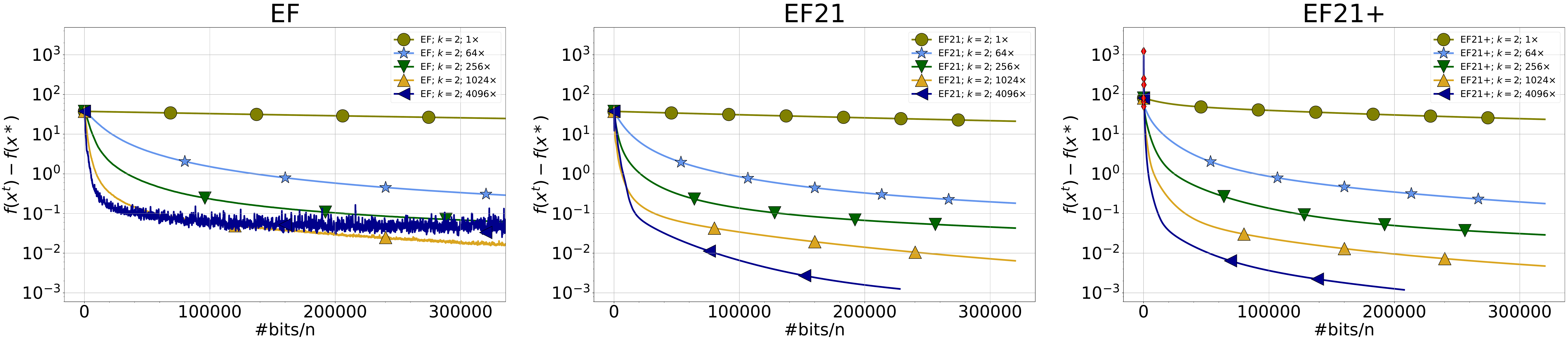}
	\includegraphics[width=0.8\linewidth]{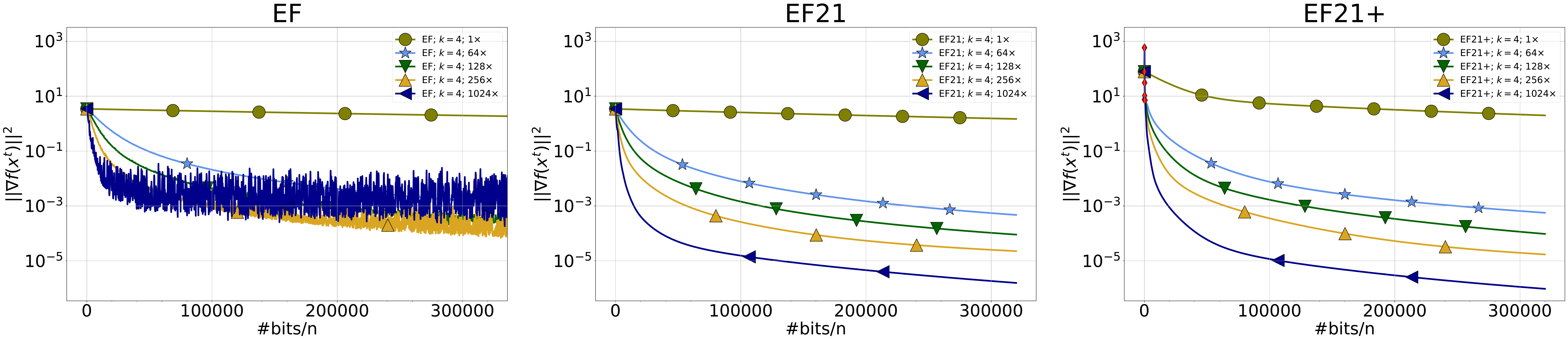}
	\includegraphics[width=0.8\linewidth]{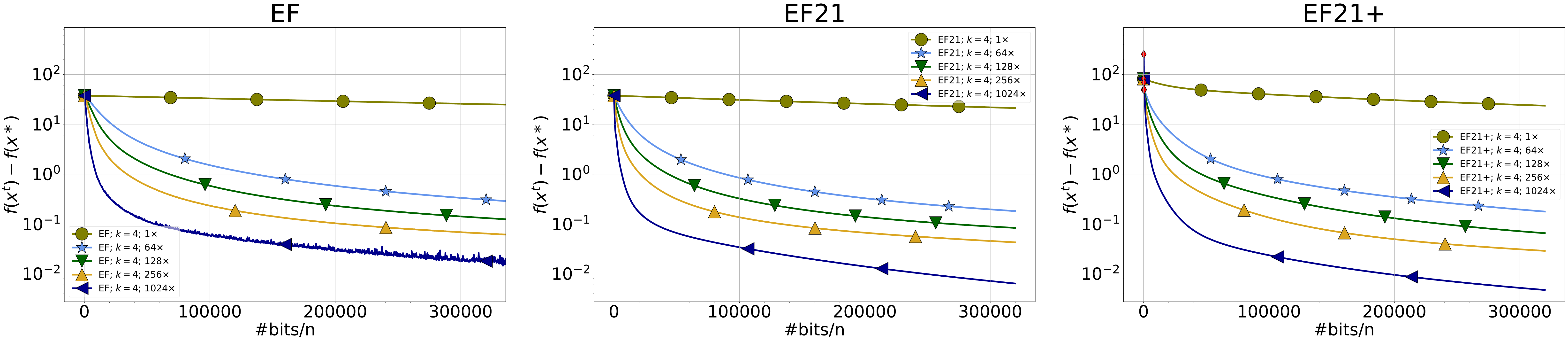}
	%%%%%%%
	%\hline
	%%%%%%%
	
	\caption{The performance of \algname{EF}, \algname{EF21}, and \algname{EF21+}  with Top-$k$ compressor, and for increasing stepsizes. The dataset used: \texttt{w8a}. By $1\times, 2\times, 4\times$ (and so on) we indicate that the stepsize was set to a multiple of  the largest stepsize predicted by our theory.} \label{fig:3plots_more_w8a_PL}
\end{figure}

All of the figures above illustrate that in the PL setting, \algname{EF21} and \algname{EF21+} tolerate much larger stepsizes than \algname{EF}, which makes them more efficient in practice. Moreover, in all experiments with large stepsizes ($512\times$--$4096\times$), \algname{EF} starts oscillating, which hinders the convergence to the desired tolerance.

\subsection{Deep learning experiments}\label{sec:exp-DL}

In this section, we replace full gradient $\nabla f_i(x^{k+1})$ in the algorithms \algname{EF21} and \algname{EF} by its stochastic estimator (minibatch without replacement), and conduct several deep learning experiments for multi-class image classification. In particular, we compare our \algname{EF21} method to \algname{EF} by running  ResNet18 \citep{he2016deep} and VGG11 models on the CIFAR-10 \citep{krizhevsky2009learning} dataset. 

We implement the algorithms in PyTorch \citep{paszke2019pytorch} and run the experiments on several GPUs. We used 3 different GPU cluster node types in total within all experiments: \begin{enumerate}
\item NVIDIA GeForce GTX 1080 Ti; \item NVIDIA GeForce RTX 2080 Ti; \item  NVIDIA Tesla V100. \end{enumerate}

The dataset is split into $n = 5$ equal parts. Total train set size for CIFAR-10 is  $50,000$. The test set for evaluation has $10,000$ data points. The train set is split into batches of size $\tau \in \cb{128, 1024}$. The first four workers own equal number of batches of data, while the last worker has the rest. 

\subsubsection{Tuned stepsizes} 

%In our first experiments  Figures \ref{fig:DL}, \ref{fig:DL_vgg11} we fix $k \approx 0.05D $, and $\tau = 1024$ for ResNet18 and $\tau = 128$ for VGG11.\footnote{$D$ is a number of model parameters. For ResNet18, $D = 11,511,784$, and for VGG11, $D = 132,863,336$.} We tune stepsize starting from $10^{-3}$ as a baseline and increase it by a factor of $2$. In Figure \ref{fig:DL} we compare \algname{EF}, \algname{EF21}, \algname{EF21+}, and \algname{SGD} with the best tuned stepsizes. The experiment shows that during the training, both \algname{EF} and \algname{EF21} (\algname{EF21+}) perform similarly with a slight improvement in the new \algname{EF21} method. Moreover, \algname{EF21} achieves better test accuracy for both NN architechtures.
% 
%	\begin{figure}[H] 
%			\includegraphics[width=\linewidth]{plots/9_nice_plot_bsz_1024_K_4} 
%		\caption{ResNet18 on CIFAR-10.}
%		\label{fig:DL}
%	\end{figure}

In our first experiments, summarized in  Figures \ref{fig:DL} and \ref{fig:DL_vgg11}, we fix $k \approx 0.05D $ and $\tau = 1024$ for ResNet18, and $\tau = 128$ for VGG11.\footnote{$D$ is the number of model parameters. For ResNet18, $D = 11,511,784$, and for VGG11, $D = 132,863,336$.} We tune the stepsize starting from $10^{-3}$ as a baseline, and progressively increase it by a factor of $2$. In Figure~\ref{fig:DL} we compare \algname{EF}, \algname{EF21}, \algname{EF21+}, and \algname{SGD} with the best tuned stepsizes. The experiment shows that during the training, both \algname{EF} and \algname{EF21} (\algname{EF21+}) perform similarly with a slight improvement in the new \algname{EF21} method. Moreover, \algname{EF21} achieves better test accuracy for both NN architectures.

\begin{figure}[H] 
	\includegraphics[width=\linewidth]{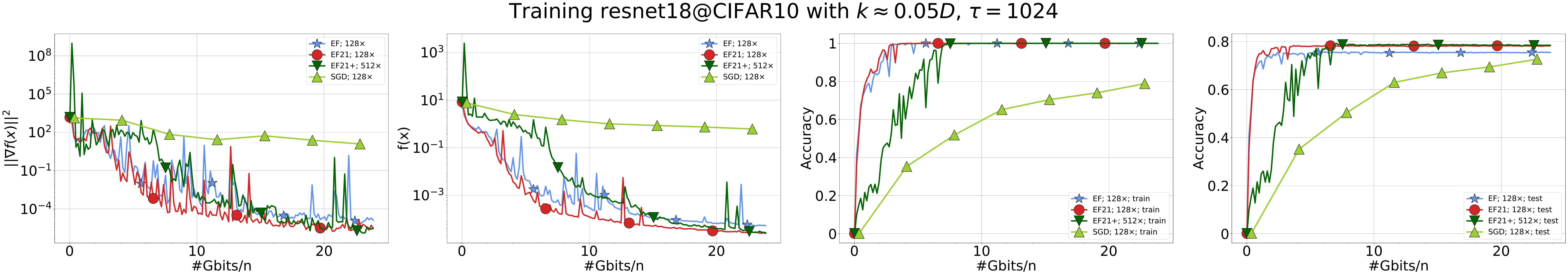} 
	\caption{ResNet18 on CIFAR-10.}
	\label{fig:DL}
\end{figure}

\begin{figure}[H] 
	\includegraphics[width=\linewidth]{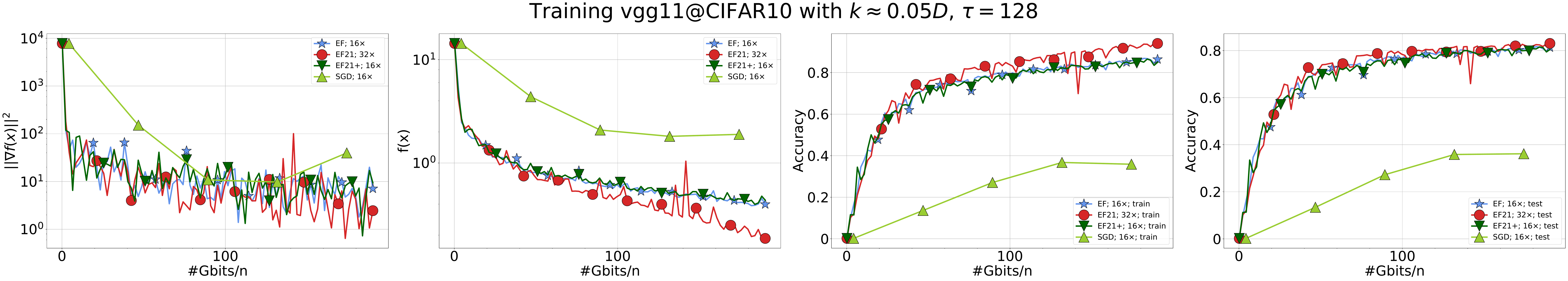} 
	\caption{VGG11 on CIFAR-10.}
	\label{fig:DL_vgg11}
\end{figure}

\subsubsection{Dependence on $k$} 
In this experiment, we fix the batch size $\tau = 1024$ and a medium stepsize $\g = 1.6 \cdot 10^{-2}$. We demonstrate that choosing smaller $k$ in the Markov compressor makes the method more communication efficient, and helps it to more quickly achieve higher test accuracy.

\begin{figure}[H] 
	\includegraphics[width=\linewidth]{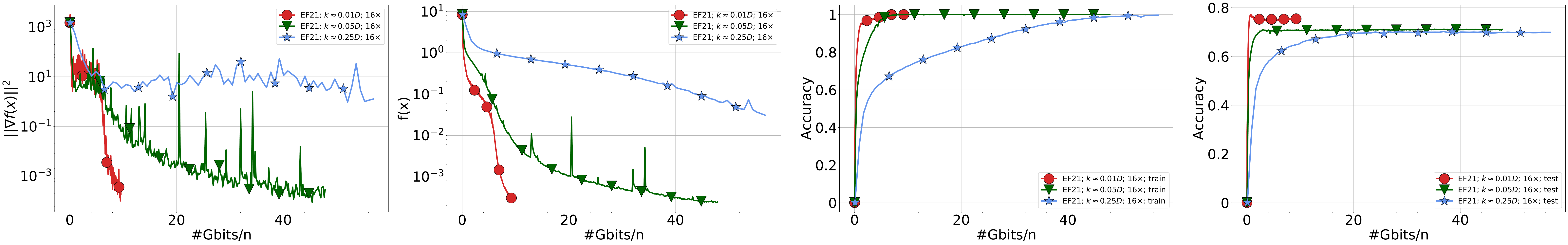} 
	\caption{ResNet18 on CIFAR-10, minibatch size $\tau = 1024$.}
	\label{fig:DL_vary_K}
\end{figure}

		%%%%%%%%%%%%%%%%%%%%%%%%
	%%%%%%%%%%%%%%%%%%%%%%%%
		% \clearpage		
		\section{Proofs for Section~\ref{sec:Markov}: Distortion of Markov Compressor} \label{sec:markov-proofs}
	%%%%%%%%%%%%%%%%%%%%%%%%
	%%%%%%%%%%%%%%%%%%%%%%%%	
	
We have made a couple statements, without proof, at the end of Section~\ref{sec:Markov} which were not critical to the development of our results. Here we provide the justification.
	
	\begin{lemma} \label{lem:Markov-property}Let $\{v^t\}_{t\geq 0}$ be any sequence of vectors in $\R^d$. Let  \begin{equation}D^t\eqdef \sqnorm{\cM(v^{t+1}) - v^t}\label{eq:D^t}\end{equation} be the distortion of the Markov compressor $\cM$ on input $v^t$. Then
	\begin{equation}\Exp{D^t} \leq (1-\theta)^t \Exp{D^0} + \beta \sum_{i=0}^{t-1} (1-\theta)^i \Delta^{t-i},\label{eq:u9h98gfdfddf}\end{equation}
	where $\Delta^t\eqdef \sqnorm{v^{t+1} - v^t}$.
	\end{lemma}
	\begin{proof} By conditioning on $\cM(v^t)$, we get
	\begin{eqnarray} 
	\Exp{D^{t+1} \mid  \cM(v^{t})} &=& \Exp{\sqnorm{\cM(v^{t+1}) - v^{t+1}} \mid  \cM(v^{t}) }  \notag \\
	&=& \Exp{ \sqnorm{\cM(v^{t}) + \cC \left( v^{t+1} - \cM(v^{t})  \right) - v^{t+1}}  \mid  \cM(v^{t})}  \notag \\
	&\overset{\eqref{eq:b_compressor}}{\leq} & (1-\alpha) \sqnorm{v^{t+1} - \cM(v^{t})} \notag \\
	&\leq &(1-\alpha) \left[(1+s) \sqnorm{v^{t} - \cM(v^{t})} + (1+s^{-1})\sqnorm{v^{t+1} - v^{t}}\right] \notag \\
	&=& (1-\theta)\sqnorm{v^{t} - \cM(v^{t})} + \beta \Delta^t,\label{eq:ni8f0dh9hf0d}
 	\end{eqnarray}
	where $s>0$ is small enough so that that $1-\theta = (1-\alpha) (1+s) <1$, and we define $\beta = (1-\alpha)(1+s^{-1})$.
	
	By applying the tower property, we get
	\begin{eqnarray*}
	\Exp{D^{t+1}} &=& \Exp{\Exp{D^{t+1} \mid  \cM(v^{t})}} \\
	&\overset{\eqref{eq:ni8f0dh9hf0d}}{=}&	(1-\theta)\Exp{\sqnorm{v^{t} - \cM(v^{t})}} + \beta \Delta^t \\
	&\overset{\eqref{eq:D^t}}{=}& (1-\theta) \Exp{D^t} + \beta \Delta^t.
	 	\end{eqnarray*}
	It remains to unroll this recurrence.	
		
	\end{proof}

\begin{corollary} Assume that $\Delta^t \leq (1-\phi)^t \Delta^0$ for all $t\geq 0$ and some $\phi>0$. Then
\[\lim \limits_{t\to \infty} \Exp{D^t} = 0.\]
\end{corollary}
\begin{proof}
Using Lemma~\ref{lem:Markov-property}, we get
\begin{eqnarray*}\Exp{D^t} &\overset{\eqref{eq:u9h98gfdfddf}}{\leq} & (1-\theta)^t \Exp{D^0} + \beta \sum_{i=0}^{t-1} (1-\theta)^i \Delta^{t-i}\\
&\leq & (1-\theta)^t \Exp{D^0} + \beta\Delta^0  \sum_{i=0}^{t-1} (1-\theta)^i (1-\phi)^{t-i} \\
&\leq & (1-\theta)^t \Exp{D^0} + \beta\Delta^0  \sum_{i=0}^{t-1} (1-\min\{\theta,\phi\})^t \\
&=& (1-\theta)^t \Exp{D^0} +   t (1-\min\{\theta,\phi\})^t    \beta\Delta^0.
\end{eqnarray*}
Clearly, the right hand side converges to $0$ as $t\to \infty$.
\end{proof}

	%%%%%%%%%%%%%%%%%%%%%%%%
	%%%%%%%%%%%%%%%%%%%%%%%%
		% \clearpage
	\section{Four Lemmas Needed in the Proofs of Theorems~\ref{thm:main-distrib} and~\ref{thm:PL-main} }
	%%%%%%%%%%%%%%%%%%%%%%%%
	%%%%%%%%%%%%%%%%%%%%%%%%

	We first state several auxiliary results we need for the proofs of our main theorems.
	
    \subsection{Compression distortion bound}
    
The following lemma play a key role in our analysis. It characterizes the change of the distortion imparted by  the Markov compressor in a single iteration.

	\begin{lemma}\label{lem:theta-beta} Let $\cC\in \bB(\alpha)$ for $0<\alpha\leq 1$. Define
	$G_i^t \eqdef  \sqnorm{ g_i^t - \nabla f_i(x^t) } $ and $W^t \eqdef \{g_1^t, \dots, g_n^t, x^t, x^{t+1}\}$. 		For any $s>0$ we have
		\begin{equation}\label{eq:90y0yfhdf} \Exp{ G_i^{t+1} \;|\; W^t} \leq (1-\theta(s))   G_i^t + \beta(s)  \sqnorm{\nabla f_i(x^{t+1}) - \nabla f_i(x^t)} ,\end{equation}
		where
		\begin{equation}\label{eq:theta-beta-def}\theta(s) \eqdef 1- (1- \alpha )(1+s), \qquad \text{and} \qquad \beta(s)\eqdef (1- \alpha ) \left(1+ s^{-1} \right).\end{equation}

		\end{lemma}

 	\begin{proof}
 		\begin{eqnarray*}
 			\Exp{ G_i^{t+1} \;|\; W^t} & = & \Exp{  \sqnorm{g_i^{t+1} - \nabla f_i(x^{t+1})}  \;|\; W^t}	 \\
 			&=& \Exp{  \sqnorm{g_i^t + \cC ( \nabla f_i(x^{t+1}) - g_i^t) - \nabla f_i(x^{t+1})}  \;|\; W^t}	 \\
 			&\overset{\eqref{eq:b_compressor}}{\leq} &  (1-\alpha) \sqnorm{\nabla f_i(x^{t+1}) - g_i^t} \\
			&\leq & (1-\alpha)  (1+ s) \sqnorm{\nabla f_i(x^{t}) - g_i^t}  \\
			&& \qquad  + (1-\alpha)  \left(1+s^{-1}\right) \sqnorm{\nabla f_i(x^{t+1}) - \nabla f_i(x^t)} , 		\end{eqnarray*}
			where the last inequality follows from Young's inequality, which states that for any $a,b\in \R^d$ and any $s>0$ we have $\sqnorm{a+b}\leq (1+s)\sqnorm{a} + (1+s^{-1})\sqnorm{b}$.			
 	\end{proof}

In particular, consider node $i$ and iteration $t$. Applying Markov compressor specific to node $i$ (let us call it $\cM_i$) to $v_i^t = \nabla f_i(x^t)$, we get $g_i^t = \cM_i(v_i^t)$. In the next iteration, we apply Markov compressor to the new gradient, $v_i^{t+1} = \nabla f_i(x^{t+1})$, and the compressed vector is $g_i^{t+1} = \cM_i(v_i^{t+1})$. Note that $G_i^t$ is the distortion of Markov compressor at iteration $t$, and that \eqref{eq:90y0yfhdf} describes how this distortion changes from iteration $t$ to iteration $t+1$.  The expectation on the left hand side is over the randomness inherent in $\cC$ (and so, for example, if $\cC$ is the Top-$k$ compressor, expectation is not needed). 

Note that since the distortion of the  Markov compressor at iteration $t$ is equal to $$G_i^t \eqdef  \sqnorm{ g_i^t - \nabla f_i(x^t)} ,$$ \eqref{eq:90y0yfhdf} says that, provided that $\theta(s)>0$, the distortion decreases by the factor of $1-\theta(s)$, subject to  the additive error $$\varepsilon_i^t(s) \eqdef \beta(s)  \sqnorm{\nabla f_i(x^{t+1}) - \nabla f_i(x^t)}.$$ That is, \eqref{eq:90y0yfhdf} can be written in the form
\[ \Exp{ \sqnorm{\cM_i(\nabla f_i (x^{t+1})) - \nabla f_i (x^{t+1})}\mid W^t} \leq  (1-\theta(s))  \sqnorm{\cM_i(\nabla f_i (x^{t})) - \nabla f_i (x^{t})} + \varepsilon_i^t(s).\]

Note that since our method converges, the difference $\nabla f_i(x^{t+1}) - \nabla f_i(x^t)$ decreases to zero, and hence the additive error $\varepsilon_i^t(s)$ decreases to zero, too.

Note that the distortion evolution mechanism described by Lemma~\ref{lem:theta-beta} is fundamentally different from the distortion evolution mechanism behind the vanilla biased compressor $\cC$. Indeed, for this compressor we instead have 
\[ \Exp{ \sqnorm{\cC(\nabla f_i (x^{t+1})) - \nabla f_i (x^{t+1})}\mid W^t} \leq  (1-\alpha)  \sqnorm{ \nabla f_i (x^{t+1})} .\]
This inequality bounds the distortion, but does not provide a {\em recursion} characterizing how the distortion changes from one iteration to another.

	\subsection{Optimal choice of $s$ in Lemma~\ref{lem:theta-beta}}
	
Notice that in  Lemma~\ref{lem:theta-beta} we have some freedom in how to choose $s$. It turns out, and this will be apparent from the proofs of Theorems~\ref{thm:main-distrib} and \ref{thm:PL-main}, that the optimal way of choosing $s$ is to minimize the ratio $\frac{\beta(s)}{\theta(s)}$. The next lemma characterizes the optimal choice of $s$. Note that the upper bound on $s$ is equivalent to requiring that $\theta(s)>0$, i.e., that the first term on the right hand side in \eqref{eq:90y0yfhdf} results in a contraction.
	
	\begin{lemma}\label{le:optimal_t-Peter} Let $0<\alpha\leq 1$ and for $s>0$ let $\theta(s)$ and $\beta(s)$ be as in \eqref{eq:theta-beta-def}. Then the solution of the optimization problem
	\begin{equation}\label{eq:98g_(89fd8gf9d} \min_{s} \left\{ \frac{\beta(s)}{\theta(s)} \;:\; 0<s<\frac{\alpha}{1-\alpha}\right\}\end{equation}
	is given by $s^* = \frac{1}{\sqrt{1-\alpha}}-1$. Furthermore, $\theta(s^*) = 1-\sqrt{1-\alpha}$, $\beta(s^*) = \frac{1-\alpha}{1-\sqrt{1-\alpha}}$ and
	\begin{equation}\label{eq:n98fhgd98hfd}\sqrt{\frac{\beta(s^*)}{\theta(s^*)}} = \frac{1}{\sqrt{1-\alpha}} -1 = \frac{1}{\alpha} + \frac{\sqrt{1-\alpha}}{\alpha} - 1 \leq \frac{2}{\alpha}-1.\end{equation}
\end{lemma}
\begin{proof}
After simple algebraic manipulation, it is easy to see that
\[\frac{\beta(s)}{\theta(s)}  =  \left( \frac{1}{1-\alpha} - \frac{1}{(1+s)(1-\alpha)} - s\right)^{-1},\]
and hence the optimization problem \eqref{eq:98g_(89fd8gf9d} is equivalent to the problem
\[ \min_{s} \left\{ \varphi(s)\eqdef \frac{1}{(1+s)(1-\alpha)} + s  \;:\; 0<s<\frac{\alpha}{1-\alpha}  \right\}.\]
Note that $\varphi$ is convex, and that $\varphi(0) = \varphi(\frac{\alpha}{1-\alpha} ) = \frac{1}{1-\alpha} $. Hence, the global minimum of $\varphi$ must lie in the interval $0<s<\frac{\alpha}{1-\alpha} $. Thus, we can drop the constraints, and find the solution by looking for a stationary point (i.e., for $s^*$ satisfying $\varphi'(s^*)=0$), which leads to $s^* =  1-\sqrt{1-\alpha}$. The rest follows by substituting the value $s=s^*$ to the expressions for $\theta(s)$, $\beta(s)$ and $\sqrt{ \frac{\beta(s)}{\theta(s)}}$.
\end{proof}

\subsection{A descent lemma}

The next lemma, due to \citet{PAGE2021}, gives a bound on the function value after one step of a method of the type $$x^{t+1}\eqdef x^{t}-\g g^{t} ,$$ where $g^t\in \R^d$ is any vector, and $\g>0$ any scalar. The only assumption we need for it to hold is for $f$ to have $L$-Lipschitz gradient.
	
		\begin{lemma}[\citep{PAGE2021}]\label{lem:80hf08dgfd}
		Suppose that function $f$ is $L$-smooth and let $x^{t+1}\eqdef x^{t}-\g g^{t} ,$ where $g^t\in \R^d$ is any vector, and $\g>0$ any scalar. Then we have
		\begin{eqnarray}\label{eq:aux_smooth_lemma}
		f(x^{t+1}) \leq f(x^{t})-\frac{\g}{2}\sqnorm{\nabla f(x^{t})}-\left(\frac{1}{2 \g}-\frac{L}{2}\right)\sqnorm{x^{t+1}-x^{t}}+\frac{\g}{2}\sqnorm{g^{t}-\nabla f(x^{t})}.
		\end{eqnarray}
	\end{lemma}

\subsection{Stepsize selection}
	
The only purpose of our final lemma is to get an easy-to-write bound on the stepsize. We achieve this at the cost of a slightly worse theoretical result, by at most a factor of two. In particular, in the proof of our main theorems, the stepsize needs to satisfy an inequality of the type \begin{equation}a \gamma^{2}+b \gamma \leq 1\label{eq:quadratic_ineq}\end{equation} where $a, b$ are positive scalars. Instead of writing an algebraic expression for the largest $\gamma$ satisfying this inequality (let's call this optimal stepsize $\gamma^*$), we first observe that, necessarily, 
$$\gamma^* \leq  \min \left\{\frac{1}{\sqrt{a}}, \frac{1}{b}\right\}.$$
Further, it is easy to verify that
$\gamma^- \eqdef \frac{1}{\sqrt{a}+b}$ satisfies the quadratic inequality \eqref{eq:quadratic_ineq}, and that $\gamma^+ \eqdef \frac{2}{\sqrt{a}+b}$  does not. So, any $0\leq \gamma \leq \gamma^-$ satisfies \eqref{eq:quadratic_ineq}, and the upper bound is at most a factor of 2 worse than $\gamma^*$. 

We now formalize the above observations.
		
	\begin{lemma}\label{le:stepsize_page_fact}
		Let $a,b>0$. If $0 \leq \gamma \leq \frac{1}{\sqrt{a}+b}$, then $a \gamma^{2}+b \gamma \leq 1$. Moreover, the bound is tight up to the factor of 2 since $\frac{1}{\sqrt{a}+b} \leq \min \left\{\frac{1}{\sqrt{a}}, \frac{1}{b}\right\} \leq \frac{2}{\sqrt{a}+b}$
	\end{lemma}

%	\begin{lemma}\label{lemma:main_recursion_distrib1}
%		Let $G^t \eqdef \frac{1}{n} \sum_{i=1}^n \sqnorm{ g_i^t - \nabla f_i(x^t) } $ and $W^t \eqdef \{g_1^t, \dots, g_n^t, x^t, x^{t+1}\}$. Then for any $t>0$ we have
%		\begin{equation}\label{eq:98h98fdh98f} \Exp{ G^{t+1} \;|\; W^t} \leq  \underbrace{(1- \alpha )(1+t)}_{1-\theta}   G^t + \underbrace{(1- \alpha ) \left(1+ t^{-1} \right)}_{\beta}  \tilde{L}^2 \sqnorm{x^{t+1} -x^{t} } .\end{equation}
%	\end{lemma}
%	\begin{proof}
%		\begin{eqnarray*}
%			\Exp{ G^{t+1} \;|\; W^t} & = & \frac{1}{n} \sum_{i=1}^n  \Exp{  \sqnorm{g_i^t + \cC ( \nabla f_i(x^{t+1}) - g_i^t) - \nabla f_i(x^{t+1})}  \;|\; W^t}	 \\
%			&\leq & \frac{1}{n} \sum_{i=1}^n (1-\alpha) \sqnorm{\nabla f_i(x^{t+1}) - g_i^t} \\
%			&\leq & (1-\alpha)  \frac{1}{n} \sum_{i=1}^n \left[ (1+ t) \sqnorm{\nabla f_i(x^{t}) - g_i^t}  + \left(1+t^{-1}\right) \sqnorm{\nabla f_i(x^{t+1}) - \nabla f_i(x^t)}\right] \\	 
%			& =& (1-\alpha)  (1+t) G^t  + (1-\alpha) \left(1+t^{-1}\right) \left( \frac{1}{n} \sum_{i=1}^n L_i^2 \right) \sqnorm{x^{t+1} - x^t} .	 	 	 
%		\end{eqnarray*}
%	\end{proof}
%

	%%%%%%%%%%%%%%%%%%%%%%%%
	%%%%%%%%%%%%%%%%%%%%%%%%
	% \clearpage
\section{Proof of Theorem~\ref{thm:main-distrib}}\label{proof:main-distrib}
	%%%%%%%%%%%%%%%%%%%%%%%%
	%%%%%%%%%%%%%%%%%%%%%%%%
		
\begin{proof}

{\bf STEP 1.} Recall that Lemma~\ref{lem:theta-beta} says that			\begin{equation}\label{eq:n89fg9d08hfbdi_8f}
				\Exp{\sqnorm{g_i^{t+1} -  \nabla f_i(x^{t+1})}\mid W^t } \overset{\eqref{eq:90y0yfhdf}}{\leq} (1 - \theta)\sqnorm{g_i^t - \nfixt} + \beta  \sqnorm {\nfixtpo - \nfixt} ,
			\end{equation}
			where
$\theta = \theta(s^*)$ and $\beta = \beta(s^*)$ are given by Lemma~\ref{le:optimal_t-Peter}.			 Averaging inequalities \eqref{eq:n89fg9d08hfbdi_8f} over $i \in \{1,2,\dots,n\}$ gives
			\begin{eqnarray}
				\Exp{G^{t+1} \mid W^t } &\overset{\eqref{eq:G^t}}{=}&  \frac{1}{n}\sumin\Exp{\sqnorm{ g_i^{t+1} - \nfixtpo} \mid W^t} \notag \\
				&\overset{\eqref{eq:n89fg9d08hfbdi_8f}}{\leq}& \rb{1 - \theta}\suminn \sqnorm{g_i^t - \nfixt} + \beta \suminn \sqnorm {\nfixtpo - \nfixt} \notag \\
				&\overset{\eqref{eq:G^t}}{=}& \rb{1 - \theta}G^t+ \beta \suminn \sqnorm {\nfixtpo - \nfixt} \notag \\
				&\le& \rb{1 - \theta}G^t+ \beta \rb{\suminn L_i^2}\sqnorm{x^{t+1} - x^t} .\label{eq:jbiu-9u0df9}
			\end{eqnarray}
			Using Tower property and $L$-smoothness in \eqref{eq:jbiu-9u0df9}, we proceed to 
			\begin{eqnarray}\label{eq:main_recursion_distrib}
				\Exp{G^{t+1}} = \Exp{\Exp{G^{t+1} \mid W^t}} \overset{\eqref{eq:jbiu-9u0df9}}{\leq} \rb{1 - \theta}\Exp{G^t}+ \beta \wL^2 \Exp{ \sqnorm{x^{t+1} - x^t}}.
			\end{eqnarray}
			
{\bf STEP 2.} Next, using Lemma~\ref{lem:80hf08dgfd} and Jensen's inequality applied to the function $x\mapsto \sqnorm{x}$, we obtain the bound 
			\begin{eqnarray}
				f(x^{t+1}) &\letext{\eqref{eq:aux_smooth_lemma}}& f(x^{t})-\frac{\g}{2}\sqnorm{\nabla f(x^{t})}-\left(\frac{1}{2 \g}-\frac{L}{2}\right)\sqnorm{x^{t+1}-x^{t}}+\frac{\g}{2}\sqnorm{\suminn \rb{g_i^t-\nabla f_i(x^{t}) }}\notag \\
				&\overset{\eqref{eq:G^t}}{\leq}&
				f(x^{t})-\frac{\g}{2}\sqnorm{\nabla f(x^{t})}-\left(\frac{1}{2 \g}-\frac{L}{2}\right)\sqnorm{x^{t+1}-x^{t}}+\frac{\g}{2} G^t. \label{eq:aux_smooth_lemma_distrib}
			\end{eqnarray}

			Subtracting $f^{\text {inf }}$ from both sides of \eqref{eq:aux_smooth_lemma_distrib} and taking expectation, we get
			\begin{eqnarray}\label{eq:func_diff_distrib}
				\Exp{f(x^{t+1})-\finf} &\leq& \quad \Exp{f(x^{t})-\finf}-\frac{\gamma}{2} \Exp{\sqnorm{\nabla f(x^{t})}} \notag \\
				&& \qquad -\left(\frac{1}{2 \gamma}-\frac{L}{2}\right) \Exp{\sqnorm{x^{t+1}-x^{t}}}+ \frac{\gamma}{2}\Exp{G^t}.
			\end{eqnarray}

{\bf COMBINING STEP 1 AND STEP 2.}	
	 Let $\delta^{t} \eqdef \Exp{f(x^{t})-\finf}$, $s^{t} \eqdef \Exp{G^t }$ and $r^{t} \eqdef\Exp{\sqnorm{x^{t+1}-x^{t}}}.
			$
			Then by adding \eqref{eq:func_diff_distrib} with a $\frac{\gamma}{2 \theta}$ multiple of \eqref{eq:main_recursion_distrib} we obtain
			\begin{eqnarray*}
				\delta^{t+1}+\frac{\gamma}{2 \theta} s^{t+1} &\leq& \delta^{t}-\frac{\gamma}{2}\sqnorm{\nabla f(x^{t})}-\left(\frac{1}{2 \gamma}-\frac{L}{2}\right) r^{t}+\frac{\gamma}{2} s^{t}+\frac{\gamma}{2 \theta}\left(\beta \wL^2 r^t + (1 - \theta) s^{t}\right) \\
				&=&\delta^{t}+\frac{\gamma}{2\theta} s^{t}-\frac{\gamma}{2}\sqnorm{\nabla f(x^{t})}-\left(\frac{1}{2\g} -\frac{L}{2} - \frac{\g}{2\theta}\beta \wL^2 \right) r^{t} \\
				& \leq& \delta^{t}+\frac{\gamma}{2\theta} s^{t} -\frac{\gamma}{2}\sqnorm{\nabla f(x^{t})}.
			\end{eqnarray*}
			The last inequality follows from the bound $\g ^2\frac{\beta \wL^2}{\theta} + L\g \leq 1,$ which holds because of Lemma \ref{le:stepsize_page_fact} and our assumption on the stepsize.
			By summing up inequalities for $t =0, \ldots, T-1,$ we get
			$$
			0 \leq \delta^{T}+\frac{\gamma}{2 \theta} s^{T} \leq \delta^{0}+\frac{\gamma}{2 \theta} s^{0}-\frac{\gamma}{2} \sum_{t=0}^{T-1} \Exp{\sqnorm{\nabla f(x^{t})}}.
			$$
			Multiplying both sides by $\frac{2}{\gamma T}$, after rearranging we get
			$$
			\sum_{t=0}^{T-1} \frac{1}{T} \Exp{\sqnorm{\nabla f (x^{t})}} \leq \frac{2 \delta^{0}}{\gamma T} + \frac{s^0}{\theta T}.
			$$
			It remains to notice that the left hand side can be interpreted as $\mathrm{E}\left[\sqnorm{\nabla f(\hat{x}^{T})}\right]$, where $\hat{x}^{T}$ is chosen from $x^{0}, x^{1}, \ldots, x^{T-1}$ uniformly at random.			\end{proof}

	%%%%%%%%%%%%%%%%%%%%%%%%
	%%%%%%%%%%%%%%%%%%%%%%%%
		% \clearpage	
	\section{Proof of Theorem~\ref{thm:PL-main} }\label{proof:PL-main}
	%%%%%%%%%%%%%%%%%%%%%%%%
	%%%%%%%%%%%%%%%%%%%%%%%%
	
	\begin{proof}
		We proceed as in the previous proof, but use the PL inequality and subtract $f(x^\star)$ from both sides of  \eqref{eq:aux_smooth_lemma_distrib} to get
		\begin{eqnarray*}
			\Exp{f(x^{t+1})- f(x^\star)}  &\letext{\eqref{eq:aux_smooth_lemma_distrib}} &  \Exp{ f(x^{t})- f(x^\star) }-\frac{\g}{2}\sqnorm{\nabla f(x^{t})}-\left(\frac{1}{2 \g}-\frac{L}{2}\right)\sqnorm{x^{t+1}-x^{t}}+\frac{\g}{2} G^t \\ 
			& \leq & (1-\gamma \mu) \Exp{ f(x^{t})- f(x^\star) } - \left(\frac{1}{2 \g}-\frac{L}{2}\right)\sqnorm{x^{t+1}-x^{t}}+\frac{\g}{2} G^t .
		\end{eqnarray*}
		
		Let $\delta^{t} \eqdef \Exp{f(x^{t})-f(x^\star)}$, $s^{t} \eqdef \Exp{G^t }$ and $ r^{t} \eqdef \Exp{\sqnorm{x^{t+1}-x^{t}}}$. Then by adding the above inequality with a $\frac{\gamma}{\theta}$ multiple of \eqref{eq:main_recursion_distrib}, we obtain
		\begin{eqnarray*}
			\delta^{t+1}+\frac{\gamma}{ \theta} s^{t+1} &\leq& (1-\gamma \mu)\delta^{t} -\left(\frac{1}{2 \gamma} - \frac{L}{2}\right) r^{t} + \frac{\gamma}{2} s^{t} + \frac{\gamma}{ \theta}\left( (1-\theta) s^t + \beta \wL^2 r^t       \right) \\
			&=&  (1-\gamma \mu) \delta^{t}  + \frac{\gamma}{\theta} \left(1-\frac{\theta}{2}\right) s^t  -\left(\frac{1}{2 \gamma} - \frac{L}{2} - \frac{\beta \wL^2 \gamma}{\theta}\right) r^{t}  	.
		\end{eqnarray*}
		Note that our assumption on the stepsize implies that 	$ 1 - \frac{\theta}{2} \leq 1 -\gamma \mu$ and $\frac{1}{2 \gamma} - \frac{L}{2} - \frac{\beta \wL^2 \gamma}{\theta} \geq 0$. The last inequality follows from the bound $\gamma^2\frac{2\beta \wL^2}{\theta} + \gamma L \leq 1,$ which holds because of Lemma \ref{le:stepsize_page_fact} and our assumption on the stepsize. Thus,
		\begin{eqnarray*}
			\delta^{t+1}+\frac{\gamma}{ \theta} s^{t+1} &\leq& (1-\gamma \mu) \left(\delta^{t}+\frac{\gamma}{ \theta} s^{t} \right).
		\end{eqnarray*}
		It remains to unroll the recurrence.
	\end{proof}

	%%%%%%%%%%%%%%%%%%%%%%%%%%	
	%%%%%%%%%%%%%%%%%%%%%%%%%%	
	%%%%%%%%%%%%%%%%%%%%%%%%%%	
	% \clearpage
	\section{Dealing with Stochastic Gradients (Details for Section~\ref{sec:stoch_grad-Main})} \label{sec:extensions}
	%%%%%%%%%%%%%%%%%%%%%%%%%%	
	%%%%%%%%%%%%%%%%%%%%%%%%%%	
	%%%%%%%%%%%%%%%%%%%%%%%%%%	
	
% Here we present some a few extensions and possible extensions which we did not mention in the main paper.

We now describe a natural extension of \algname{EF21} to the setting where full gradient computations are replaced by stochastic gradient estimators, i.e., we use a random vector
$$\hat{g}_i^t \approx	\nabla f_i(x^t)$$
instead of $\nabla f_i(x^t)$. This simple change leads to Algorithm~\ref{alg:EF21-stoch}, where we {\color{red}highlight in red} the parts that differ from the exact/full gradient version of \algname{EF21}.

				\begin{algorithm}[h]
		\centering
		\caption{\algname{EF21} (Multiple nodes + {\color{red}Stochastic regime})}\label{alg:EF21-stoch}
		\begin{algorithmic}[1]
			\State \textbf{Input:} starting point $x^{0} \in \R^d$;  $g_i^0 = \cC({\color{red} \hat{g}_i^0 } )$,  where { \color{red}$\hat{g}_i^0 \approx \nabla f_i(x^0)$} for $i=1,\dots, n$ (known by nodes and the master); learning rate $\gamma>0$; $g^0 = \frac{1}{n}\sum_{i=1}^n g_i^0$ (known by master)
			\For{$t=0,1, 2, \dots , T-1 $}
			\State Master computes $x^{t+1} = x^t - \gamma g^t$ and broadcasts $x^{t+1}$ to all nodes
			\For{{\bf all nodes $i =1,\dots, n$ in parallel}}
			\State {\color{red}Compute a stochastic gradient $\hat{g}_i^{t+1}\approx \nabla f_i(x^{t+1})$}
			\State Compress $c_i^t = \cC({\color{red}\hat{g}_i^{t+1}} - g_i^t)$ and send $c_i^t $ to the master
			\State Update local state $g_i^{t+1} = g_i^t + \cC( {\color{red}\hat{g}_i^{t+1}} - g_i^t)$
			\EndFor
			\State Master computes $g^{t+1} = \frac{1}{n} \sum_{i=1}^n  g_i^{t+1}$ via  $g^{t+1} = g^t + \frac{1}{n} \sum_{i=1}^n c_i^t $
			\EndFor
		\end{algorithmic}
	\end{algorithm}

An analysis of this extension/generalization can be done in a similar manner. The key change is the replacement of Lemma~\ref{lem:theta-beta} in the proofs of the two complexity theorems, and then accounting for this change in the proof. However, this is easy to do. We now describe what  Lemma~\ref{lem:theta-beta}  should be replaced with.

We first start with a technical lemma.
	\begin{lemma}\label{lem:b8787fdgdf}Let $\cC\in \bB(\alpha)$, and let $\xi\in \R^d$ be a random vector independent of $\cC$, with zero mean and variance bounded as $\Exp{\sqnorm{\xi}} \leq \sigma^2$. Then for any $s>0$, we have
	\[ \Exp{ \sqnorm{ \cC( x + \xi) - x } } \leq (1-\alpha)(1+s) \sqnorm{x} + \left((1-\alpha)(1+s) + 1 + s^{-1} \right)\sigma^2, \qquad \forall x \in \R^d. \]
%	In particular, if $t=\frac{1}{2}\frac{\alpha}{1-\alpha}$, then 
%\[ \Exp{ \sqnorm{ \cC( x + \xi) - x } } \leq \left(1- \frac{\alpha}{2}\right) \sqnorm{x} + \left((1-\alpha)(1+t) + 1 + t^{-1} \right)\sigma^2. \]	
	\end{lemma}
	\begin{proof}
	First, due to Young's inequality, for any $s>0$ we have
	\begin{equation} \sqnorm{ \cC( x + \xi) - x } \leq (1+t)\sqnorm{ \cC( x + \xi) - (x + \xi) } + (1+s^{-1}) \sqnorm{\xi} .\label{eq:ih8908hgdfdff}\end{equation}
	By taking conditional expectation, we get
	\begin{eqnarray}
	\Exp{ \sqnorm{ \cC( x + \xi) - x } \mid \xi }	&\overset{\eqref{eq:ih8908hgdfdff}}{\leq} & (1+s) \Exp{\sqnorm{ \cC( x + \xi) - (x + \xi)} \mid \xi} + (1+s^{-1}) \sqnorm{\xi}  \notag \\
	&\overset{\eqref{eq:b_compressor}}{\leq}& (1+s) (1-\alpha)\sqnorm{x+ \xi } + (1+s^{-1}) \sqnorm{\xi} \notag\\
	&=& (1-\alpha)(1+s)  \sqnorm{x} + 2(1-\alpha)(1+s) \langle x, \xi \rangle \notag \\
	&& \qquad + \left((1-\alpha)(1+s) + 1 + s^{-1} \right) \sqnorm{\xi}.\label{eq:u8g9f--0fd8yhfd}
	\end{eqnarray}

 	Taking expectation again,  applying the tower property, and using the fact that $\Exp{\xi}=0$ and $\Exp{\sqnorm{\xi}} \leq \sigma^2$, we finally get
	\begin{eqnarray*} \Exp{ \sqnorm{ \cC( x + \xi) - x } } &=& \Exp{\Exp{ \sqnorm{ \cC( x + \xi) - x } \mid \xi }} \\
	&\overset{\eqref{eq:u8g9f--0fd8yhfd}}{\leq }& (1-\alpha)(1+s)  \sqnorm{x} +  \left((1-\alpha)(1+s) + 1 + s^{-1} \right) \sigma^2.
	\end{eqnarray*}

%	 \Exp{ \Exp{ \sqnorm{ \cC( x + \xi) - x } } \mid \xi} \\

	\end{proof}
	
We will choose $s<\frac{\alpha}{1-\alpha}$, so that $1-\hat{\alpha} \eqdef (1-\alpha)(1+s)<1$.
 The above lemma postulates that for $\cC\in \bB(\alpha)$, and under certain assumptions on the noise $\xi$, there exist constants $\hat{\alpha}>0$ and $\hat{\sigma}>0$ such that
\begin{equation}\label{eq:Lbi87fg8fd_08y9f8d} \Exp{ \sqnorm{ \cC( x + \xi) - x } } \leq \left(1-  \hat{\alpha} \right) \sqnorm{x} +\hat{\sigma}^2, \qquad \forall x \in \R^d. \end{equation}
We will elevate this inequality into an assumption because the particular values for $\hat{\alpha}$	and $\hat{\sigma}$ given by the lemma will not be tight for every compressor $\cC$, and we want to formulate our complexity results with as tight constants as possible. 

\begin{assumption} \label{ass:hat-ineq} Let $\cC:\R^d\to \R^d$ be a (possibly randomized) mapping and let $\xi\in \R^d$ be a random vector independent of $\cC$. We assume that there exist  constants $\hat{\alpha}>0$ and $\hat{\sigma}>0$ such that
\eqref{eq:Lbi87fg8fd_08y9f8d} holds for all $x\in \R^d$.

\end{assumption}

We now present an analogue of Lemma~\ref{lem:theta-beta} in the stochastic regime.
	\begin{lemma}\label{lem:theta-beta-stoch_case} Consider Algorithm~\ref{alg:EF21-stoch} and let the the stochastic estimator $\hat{g}_i^{t}$ be given by $$\hat{g}_i^{t} = \nabla f_i(x^t) + \xi_i^t,$$ where
	$\xi_i^t$ is a random vector. Assume that for $\xi=\xi_i^t$, inequality \eqref{eq:Lbi87fg8fd_08y9f8d} holds\footnote{Recall that by Lemma~\ref{lem:b8787fdgdf}, it holds if $\xi=\xi_i^t$ is a zero mean vector with variance bounded by $\sigma^2$.}. Let
	$G_i^t \eqdef  \sqnorm{ g_i^t - \nabla f_i(x^t) } $ and $W^t \eqdef \{g_1^t, \dots, g_n^t, x^t, x^{t+1}\}$. 		For any $t>0$ we have
		\begin{equation}\label{eq:90y0yfhdf-stoch-case} \Exp{ G_i^{t+1} \;|\; W^t} \leq  \underbrace{(1- \hat{\alpha} )(1+s)}_{1-\hat{\theta}(s)}   G_i^t + \underbrace{(1- \hat{\alpha} ) \left(1+ s^{-1} \right)}_{\hat{\beta}(s)}   \sqnorm{\nabla f_i(x^{t+1}) - \nabla f_i(x^t)}  + \hat{\sigma}^2 . 	\end{equation}
	\end{lemma}

	\begin{proof}
		\begin{eqnarray*}
			\Exp{ G_i^{t+1} \;|\; W^t} & = & \Exp{  \sqnorm{g_i^{t+1}  - \nabla f_i(x^{t+1}) }  \;|\; W^t}	 \\
&=&  \Exp{  \sqnorm{g_i^t + \cC ( \nabla f_i(x^{t+1}) + \xi_i^{t+1} - g_i^t) - \nabla f_i(x^{t+1})}  \;|\; W^t}	 \\
			&\overset{\eqref{eq:Lbi87fg8fd_08y9f8d}}{\leq} &  (1-\hat{\alpha})  \sqnorm{\nabla f_i(x^{t+1}) - g_i^t} +\hat{\sigma}^2 \\
			&\overset{}{\leq} & (1-\hat{\alpha})  (1+ s) \sqnorm{\nabla f_i(x^{t}) - g_i^t}  \\
			&& \quad + (1-\hat{\alpha})  \left(1+s^{-1}\right) \sqnorm{\nabla f_i(x^{t+1}) - \nabla f_i(x^t)} +\hat{\sigma}^2 . 		\end{eqnarray*}
	\end{proof}
	
	It is straightforward to use this inequality in the proofs of Theorems~\ref{thm:main-distrib} and \ref{thm:PL-main} to establish complexity results for our stochastic variant of \algname{EF21}.

		%%%%%%%%%%%%%%%%%%%%%%%%%%	
	%%%%%%%%%%%%%%%%%%%%%%%%%%	
	%%%%%%%%%%%%%%%%%%%%%%%%%%	
	% \clearpage
		\section{Computation of $\sqrt{\fr{\beta(s^*)}{\theta(s^*)}} $ for some Compressors}
		%%%%%%%%%%%%%%%%%%%%%%%%%%	
	%%%%%%%%%%%%%%%%%%%%%%%%%%	
	%%%%%%%%%%%%%%%%%%%%%%%%%%	
		
\subsection{From unbiased to biased compressors}

We start by proving the simple and very well known result about the relationship between the classes $\bU(\omega)$ and $\bB(\alpha)$ we mentioned in Section~\ref{sec:motivation}. 
		
	\begin{lemma}\label{le:unb_contr}
		If $\cC \in \bU(\omega)$, then $\fr{1}{1 + \omega} \cC \in \bB\rb{\fr{1}{1+\omega}}$.
	\end{lemma}
	\begin{proof}
		Fix $x \in \R^d$. 
		Note that for $\cC \in \bU(\omega)$ we have
		\begin{eqnarray}
		\Exp{\cC(x)} &=& x\label{eq:unb},\\
		\Exp{\sqnorm{\cC(x)}} &\le& (1 + \omega)\sqnorm{x}\label{eq:sqnorm_upperbound}.
		\end{eqnarray}
		Then
		\begin{eqnarray}
		\Exp{\sqnorm{\fr{1}{1 + \omega} \cC(x) - x}} &=& \fr{1}{(1 + \omega)^2} \Exp{\sqnorm{\cC(x)}} - 2\Exp{\lin{\cC(x),x}} + \sqnorm{x} \notag\\
		&\letext{\eqref{eq:unb} + \eqref{eq:sqnorm_upperbound}}& \fr{1}{1 + \omega} \sqnorm{x} - \sqnorm{x}\notag\\
		&=& \rb{1  -  \fr{1}{1 + \omega}}\sqnorm{x}.\notag
		\end{eqnarray}
	\end{proof}

		\subsection{Top-$k$ and a scaled version of Rand-$k$}
		
		We now compute the value $\sqrt{\fr{\beta(s^*)}{\theta(s^*)}} $ appearing in pour complexity theorems for two well known compressors belonging to the class $\bB(\alpha)$.
		
	\begin{example}
		Let $\cC$ be the Top-$k$ compressor. Then  $\cC\in \bB(\al)$ with $\al = \fr{k}{d}$ and $$\sqrt{\fr{\beta(s^*)}{\theta(s^*)}}  = \fr{\sqrt{1 - \nicefrac{k}{d}}}{1 - \sqrt{1 - \nicefrac{k}{d}}}.$$		
	\end{example}
	\begin{proof}
		It is well known that  $\cC \in  \bB(\al)$ with $\al = \fr{k}{d}$ (e.g., see \citep{beznosikov2020biased}). Then according to Lemma~\ref{le:optimal_t-Peter}, we have 
		\begin{eqnarray*}
		\sqrt{\fr{\beta(s^*)}{\theta(s^*)}} = \fr{\sqrt{1 - \al}}{1 - \sqrt{1 - \al}} = \fr{\sqrt{1 - \nicefrac{k}{d}}}{1 - \sqrt{1 - \nicefrac{k}{d}}} .
		\end{eqnarray*}
	\end{proof}
	
	\begin{example}
		Let $\cC = \rb{\frac{1}{1 + \omega}}\cC'$, where $\cC'$ is the Rand-$k$ compressor. Then  $\cC\in \bB(\al)$ with $\al = \fr{k}{d}$ and $$\sqrt{\fr{\beta(s^*)}{\theta(s^*)}}  = \fr{\sqrt{1 - \nicefrac{k}{d}}}{1 - \sqrt{1 - \nicefrac{k}{d}}}.$$
	\end{example}
	\begin{proof}
		It is well known that $\cC' \in  \bB(\omega)$ with $\omega = \fr{d}{k} - 1$ (e.g., see \citep{beznosikov2020biased}). Moreover, using the Lemma~\ref{le:unb_contr}, we get $\rb{\frac{1}{1 + \omega}}\cC' \in \bB\rb{\fr{k}{d}}$. Finally, according to  Lemma~\ref{le:optimal_t-Peter},  we have 
		\begin{eqnarray*}
		\sqrt{\fr{\beta(s^*)}{\theta(s^*)}} = \fr{\sqrt{1 - \al}}{1 - \sqrt{1 - \al}} = \fr{\sqrt{1 - \nicefrac{k}{d}}}{1 - \sqrt{1 - \nicefrac{k}{d}}} .
		\end{eqnarray*}
	\end{proof}

\end{document}